\documentclass[english,a4paper,11pt]{article}

\usepackage{hyperref}
\usepackage{url}

\usepackage[english]{babel}

\usepackage[utf8]{inputenc}
\usepackage{xspace}
\usepackage{amsmath,amsthm,amssymb,mathtools}
\usepackage{lmodern}

\usepackage[algo2e,ruled,vlined,linesnumbered]{algorithm2e}
\newcommand{\assign}{\leftarrow}

\usepackage{xcolor}
\usepackage{tikz}
\usepackage{graphicx}
\usepackage{enumerate}

\allowdisplaybreaks[4]
\newtheorem{theorem}{Theorem}

\theoremstyle{definition}

\newtheorem{example}[theorem]{Example}

\newcommand{\oea}{$(1 + 1)$~EA\xspace}
\newcommand{\ooeas}{$(1 + 1)$~EAs}

\newcommand{\ooea}{\oea}

\newcommand{\oplea}{$(1+\lambda)$~EA\xspace}

\newcommand{\om}{\textsc{OneMax}\xspace}
\newcommand{\onemax}{\om}
\newcommand{\lo}{\textsc{LeadingOnes}\xspace}

\newcommand{\R}{\ensuremath{\mathbb{R}}}

\newcommand{\N}{\ensuremath{\mathbb{N}}}

\newcommand{\AND}{\textsc{AND}\xspace}
\newcommand{\OR}{\textsc{OR}\xspace}
\newcommand{\NOT}{\textsc{NOT}\xspace}

\newcommand{\eps}{\varepsilon}

\newcommand{\BinVal}{\textsc{BinVal}\xspace}

\begin{document}

\title{Drift Analysis}

\author{Johannes Lengler}
\date{ETH Z\"urich, Switzerland\\[2ex]
\today
}

\maketitle

\sloppy

\begin{abstract}
Drift analysis is one of the major tools for analysing evolutionary algorithms and nature-inspired search heuristics. In this chapter we give an introduction to drift analysis and give some examples of how to use it for the analysis of evolutionary algorithms. 
\end{abstract} 

\section{Introduction}

Drift analysis goes back to the seminal paper of Hajek~\cite{hajek1982hitting}, and has since become ubiquitous in the analysis of Evolutionary Algorithms (EAs). Google Scholar lists more than 100,000 hits for the phrases ``Drift'' and ``Evolutionary Algorithm'', so a comprehensive review of all applications or even just all existing drift theorems is far beyond the scope of this chapter. Instead, the chapter serves two purposes. 

Firstly, it provides a self-contained introduction into drift analysis (Section~\ref{driftsec:elementary}), which is so far missing in the literature.\footnote{A briefer introduction can be found in~\cite{lehre2017theoretical}} This introduction is suitable for graduate students or for theory-affine researchers who have not yet encountered drift analysis. This first part will contain illustrative examples, and will discuss in detail the different requirements of the most basic drift theorems, specifically on additive drift, variable drift, and multiplicative drift. Counterexamples are given to point out when some drift theorems are not applicable, or give poor results.

Secondly, Section~\ref{driftsec:advanced} provides an overview over the most important recent developments in drift analysis, including lower and tail bounds, weak drift, negative drift, and population drift. This section is much more concise, and may also serve as a quick reference for the expert reader.

\section{Basics of Drift Analysis}

\subsection{Motivation}\label{driftsec:motivation}

To analyse the runtime of an evolutionary algorithm (or more generally, any randomised algorithm), one of the most common and successful approaches consists of the following three steps.
\begin{enumerate}[1.]
\item Identify a quantity $X_t$, the \emph{potential} (also called \emph{drift function} or \emph{distance function}), that adequately measures the progress that the algorithm has made after $t$ steps.
\item For any value of $X_t$, understand the nature of the random variable $X_{t} - X_{t+1}$, the one-step change of the potential.
\item Translate the data from step 2 into information about the runtime $T$ of the algorithm, i.e., the number of steps until the algorithm has achieved its goal. 
\end{enumerate}

Drift analysis is concerned with step 3. Generally, good drift theorems require as little information as possible about the potential $X_{t+1}$, and give as much information as possible about $T$. In the basic theorems, we only require (bounds on) the expectation $E[X_{t} - X_{t+1} \mid X_t = s]$ for all $s$, which is called \emph{drift}, in order to derive (bounds on) the expectation $E[T]$. Drift analysis has become a successful theory because the framework above is very general, and good tools for step 3 exist, which apply to a multitude of situations. In contrast, step 1 and 2 often do not generalise from one problem to another. Frequently, step 1 is the part of a runtime analysis that carries the key insight, and it usually requires much more ingenuity than the other steps. On the other hand, step 2, the analysis of $X_{t} - X_{t+1}$, requires arguably less insight. However, step 2 is usually the most lengthy and technical part of a runtime analysis. Therefore, the complexity of a proof can often be substantially reduced if only some basic information like the expectation $E[X_{t} - X_{t+1} \mid X_t =x]$ is needed in step 2.

For evolutionary algorithms, a natural candidate for the potential $X_t$ is the fitness $f(x^{(t)})$ of the best individual in the current population, especially so if the population consists only of a single individual, as for example for \ooeas. In a sense, this fitness measures the ``progress'' until time $t$ since it would exactly correspond to the quality of the output if the algorithm terminated with this generation. However, it is not necessarily the best choice to measure the progress that the algorithm has made towards finding a global optimum. For example, consider the linear fitness function\footnote{We follow the standard convention that for an $n$-dimensional vector $x$, we denote its components with $x_1, \ldots, x_n$.} $f: \{0,1\}^n$ with $f(x) = (n-1)\cdot x_1 + \sum_{i=2}^n x_i$, which puts very large emphasis on the first bit. The optimum (for maximization) is the string $x_{\text{OPT}} = (1,\ldots, 1)$, but the two strings $x_1 = (1,0,0,\ldots,0)$ and $x_2 = (0,1,1,\ldots,1)$ have the same fitness $f(x_1) = f(x_2) = n-1$. However, the string $x_2$ is much more similar to $x_{\text{OPT}}$ than $x_1$, so we should choose a potential that assigns a higher rating to $x_2$ than to $x_1$. We will see later (Example~\ref{driftexample:omlinear}) good choices for the potential in this example.

Historically, in the EA community drift analysis was preceded by the \emph{fitness level method}~\cite{wegener2003methods}. In retrospect it may be regarded as a special case of the Variable Drift Theorem that we will introduce in Section~\ref{driftsec:variable}. Likewise, the \emph{method of expected weight decrease}~\cite{neumann2007randomized} may be regarded as a predecessor of the Multiplicative Drift Theorem in Section~\ref{driftsec:multiplicative}. It is fair to say that the development of drift analysis boosted our understanding of evolutionary algorithms, either by simplifying existing results, or by achieving greater precision, or as a means to obtain qualitatively new results that may not have been achievable with the old techniques. For example, the original proof by Droste, Jansen, and Wegener that the \ooea takes time $O(n\log n)$ on all linear functions needed 7 pages~\cite{DrosteJW02}, while Doerr, Johannsen, and Winzen could reduce the proof to a single page~\cite{doerr2010drift}. To obtain the leading constant with the fitness level method would have been quite challenging and perhaps out of reach. With drift analysis, in a groundbreaking paper Witt~\cite{witt2013tight} could derive the leading constant not only for the standard mutation rate $1/n$, but for any mutation rate $c/n$, where $c$ is a constant, in a proof of $2$-$3$ pages!

\subsubsection{General Setup}\label{driftsec:situation}

Throughout this chapter we will assume that $(X_t)_{t\geq 0}$ is a sequence of non-negative random variables with a finite state space $\mathcal{S} \subseteq \R_0^+$ such that $0 \in \mathcal S$. We will denote the minimum positive state by $s_{\text{min}} := \min(\mathcal{S} \setminus \{0\})$. 
The \emph{stopping time} or \emph{hitting time of $0$} of $(X_t)_{t\geq 0}$ is defined as the smallest $t$ such that $X_t =0$.  We are generally interested in the \emph{drift} $\Delta_t(s) := E[X_{t} - X_{t+1} \mid X_t = s]$, where $t \geq 0$ and $s \in \mathcal S$. 

As with all conditional expectations, $\Delta_t(s)$ is not well-defined if $\Pr[X_t = s] = 0$. So in other words, $\Delta_t(s)$ is undefined for situations that never occur. Obviously, this is not a practical issue, and it is convenient (and common in the community) to be sloppy about such cases. So we will use phrases like ``$\Delta_t(s) \leq 1$ for all $t\geq 0$'' as a shortcut for ``$\Delta_t(s) \leq 1$ for all $t\geq 0$ for which the conditional expectation $\Delta_t(s)$ is well-defined''. 

In Section~\ref{driftsec:advanced} we will often need to work with pointwise drift and filtrations, i.e., we need to condition on the complete history (or at least the current state) of the algorithm, instead of just conditioning on the value of $X_t$. In these cases, we will denote the filtration associated with algorithm's history up to time $t$ by $\mathcal{F}_t$. Moreover, tail bounds will be formulated for a fixed initial search point $X_0 = s_0$. For details and an explanation of the technical terms ``pointwise drift'' and ``filtration'' see the corresponding paragraph in Section~\ref{driftsec:variants} below. 

Throughout the chapter, $f$ will denote a fitness function to be optimised, either maximised or minimised. For a $(1+\lambda)$-algorithm, we will use the convention that $x^{(t)}$ is the search point after $t$ generations.

\subsubsection{Variants}\label{driftsec:variants}

In the literature, terminology may vary between different authors, and there are often slightly different setups considered. We highlight some variants which occur frequently. A reader who is new to drift analysis may skip this section on first reading.

\begin{enumerate}[1.]
\item \textbf{Signs.} We consider the change $X_{t} - X_{t+1}$. In the literature, the difference is sometimes considered with opposite signs, $X_{t+1} - X_{t}$, which is arguably a more natural choice. However, since we consider drift towards zero, with our choice the drift is usually positive instead of negative. Moreover, our choice is more consistent with the established term ``negative drift'', which refers to a drift that points away from the target.
\item \textbf{Markov Chains.} Instead of \emph{any} sequence of random variables, the sequence $X_t$ is sometimes assumed to be a Markov chain, i.e., the state $X_t$ should completely determine the distribution of $X_{t+1}$. While this is a mathematically appealing scenario, it usually does not apply in the context of evolutionary algorithms. For instance, in the example from Section~\ref{driftsec:motivation} above, the information $X_t = n-1$ would tell us that the current fitness is $n-1$, but the two search points $x_1$ and $x_2$ differ in nature. Thus, the subsequent trajectory of search points depends on more information than is contained in $X_t$, and so do the subsequent potentials $X_{t+1}, X_{t+2}, \ldots$. So already in this very simple example, we do not have a Markov chain.

There are quite some papers on the theory of EAs which ignore this point, either accidentally or perhaps consciously for the sake of exposition, since Markov chains are a well-accessible concept. These papers contain drift theorems for a Markov chain $X_t$, but use them for runtime analysis in which $X_t$ is not a Markov chain. So technically speaking, the proofs are not correct. However, it is a purely technical issue: since the Markov property is not really needed for drift theorems, the derived results are still correct. An alternative was used in~\cite{lengler2016drift}, where the authors assumed an underlying Markov process $Y_t$ with arbitrary state space $\mathcal S$, and a function $\alpha: \mathcal S \to \R$. Then they formulated drift theorems for $X_t := \alpha(Y_t)$. This is a more precise description of randomised algorithms, where the internal state (e.g., the current population) is described by $Y_t$, and the real-valued potential is described by $X_t$. It has the advantage that expressions like $E[X_t - X_{t+1} \mid Y_t = s]$ are still well-defined, even if $\Pr[Y_t = s] = 0$. This is especially relevant in continuous domains. For example, assume that $Y_0$ is a real number drawn uniformly at random from $[0,1]$. Then $\Pr[Y_0 = s] = 0$ for all $s \in [0,1]$.
\item \textbf{Filtrations and Pointwise Drift.} We have defined the drift as a random variable that is conditioned on the value of $X_t$, i.e., $\Delta_t(x) = E[X_{t} - X_{t+1} \mid X_t = s]$. Instead, it is also possible to condition on the whole history of $X_t$, or even on the whole history of the algorithm. (Recall that in general, the potential $X_t$ does not completely describe the state of the algorithm at time $t$). In mathematical terms, the set of such histories is described by a \emph{filtration} of $\sigma$-algebras $\mathcal{F}_0 \subseteq \mathcal{F}_1 \subseteq \ldots$, where intuitively the $\sigma$-algebra $\mathcal{F}_t$ contains all the information that is available after the first $t$ steps of the algorithm.\footnote{Mathematically speaking, it is the coarsest $\sigma$-algebra which makes all random choices of the algorithm up to time $t$ measurable.} For example, instead of requiring that $E[X_{t} - X_{t+1} \mid X_t = s]\leq 1$ for all $t \geq 0$, we would ask that $E[X_{t} - X_{t+1} \mid F_t] \leq 1$ for all $t \geq 0$ and all histories $F_t$ up to time $t$ such that $X_t = s$ in $F_t$.\footnote{This is sometimes sloppily described by $E[X_{t} - X_{t+1} \mid X_0,\ldots,X_t]$. However, note that this is not quite correct since it only conditions on the past values of $X_t$, not on the history of the algorithm. In particular, conditioning on $X_0,\ldots,X_t$ usually does not determine the current state of the algorithm (e.g., the current search point or population).} In this case, we also speak of \emph{pointwise drift}, and we will write\footnote{by abuse of notation, for brevity} $E[X_{t} - X_{t+1} \mid \mathcal{F}_t, X_t = s] \leq 1$ to mean that for every history $F$ of the algorithm up to time $t$ with the property $X_t =s$, we have $E[X_{t} - X_{t+1} \mid F] \leq 1$.

Obviously, pointwise drift is a much stronger condition, and requiring such a strong condition in a drift theorem gives a priori a weaker theorem. However, for most applications it does not make a big difference to consider either version. Intellectually, it is arguably easier to imagine a fixed history of the algorithm, and to think about the next step in this fixed setting. Therefore, it is not uncommon in the EA community to formulate drift theorems using filtrations. However, we will also see examples (Example~\ref{driftexample:RLSLO} and~\ref{driftexample:omlinear}) where the weaker condition ``$X_t = s$'' is beneficial.

The basic drift theorems concerned with the expected runtime $E[T]$ can be formulated with either form of conditioning, and in this chapter we choose the stronger form (i.e., with weaker requirements), conditioning on $X_t = s$. However, once the drift theorems include tail bounds, things become more subtle, and it becomes essential to condition on every possible history. Therefore, we will switch to using filtrations and pointwise drift in the last part of the chapter. 
\item \textbf{Infinite Search Spaces.} We assume in this chapter that the state space $\mathcal{S}$ is finite. This makes sense in the context of this book since in discrete optimization the search spaces, and also the state spaces of the algorithms, tend to be finite (though they may be huge). However, there are problems, especially in continuous optimization, in which infinite state spaces are more natural. Generally, all drift theorems mentioned in this chapter still hold if the state space $\mathcal S \subseteq \R_0^+$ is infinite, but bounded.\footnote{Some statements like Theorems~\ref{driftthm:variable} and~\ref{driftthm:multiplicative} additionally require that the infimum $s_{\text{min}} := \inf(S \setminus \{0\})$ is strictly positive.} For unbounded search spaces, things become more complicated. The upper bounds on $E[T]$ in the drift theorems still hold in these cases, while the lower bounds on $E[T]$ fail in general~\cite{lengler2016drift}, as we will discuss briefly after Theorem~\ref{driftthm:additive}. Collections of drift theorems for unbounded spaces can be found in~\cite{kotzing2018first,lengler2016drift}.
\item \textbf{Drift Versus Expected Drift.} Unfortunately, the meaning of the term ``drift'' is somewhat inconsistent in the literature. We have defined it as the \emph{expected} change $E[X_{t} - X_{t+1} \mid X_t = s]$. However, some authors also use ``drift'' to refer to the conditional random variable $X_{t} - X_{t+1} \mid X_t = s$, and our definition would be the ``expected drift'' in their terminology. Some authors would also call the conditional expectation $E[X_{t} - X_{t+1} \mid \mathcal F_t]$ ``drift'', which is itself a random variable (by the randomness in the history of the algorithm). Again, our notion of drift would be the expected drift $E_{\mathcal F_t}\big[E[X_{t} - X_{t+1} \mid \mathcal F_t]\big]$ in this terminology. Yet another notion uses ``drift'' to refer to the conditional random variable $X_{t} - X_{t+1} \mid \mathcal F_t$. Fortunately, the heterogeneous nomenclature usually does not lead to confusion, except some minor notational irritations.
\end{enumerate}

\section{Elementary Introduction to Drift Analysis}\label{driftsec:elementary}

We start with an elementary introduction to drift analysis. We will discuss the three main workhorses, The Additive Drift Theorem~\ref{driftthm:additive}, the Variable Drift Theorem~\ref{driftthm:variable}, and the Multiplicative Drift Theorem~\ref{driftthm:multiplicative}. All of them give upper bounds on the expected hitting time $E[T]$, the Additive Drift Theorem also matching lower bounds.\footnote{Note that the expectation of a random variable may not always give the full story. There are even cases where the value of $E[T]$ may be misleading. We will discuss such examples in Section~\ref{driftsec:tailbounds}, where we consider drift theorems that give tail bounds on $T$.} 
\subsection{Additive Drift}\label{driftsec:additive}


The simplest possible drift is additive drift, i.e., $X_{t+1}$ differs from $X_t$ in expectation by an additive constant. The theorem in its modern form dates back to He and Yao \cite{HeY01, HeYao:04:drift}, who built on work by Hajek~\cite{hajek1982hitting},\footnote{They were apparently all unaware that the result had been proven even earlier by Tweedie~\cite[Theorem 6]{tweedie1976criteria}, and a yet earlier proof in Russian has been attributed to Menshikov~\cite[Bibliographical Notes on Section 2.6]{menshikov2016non}. The Additive Drift Theorem has been proven and rediscovered many times, and it is known under various names. For example, in stability theory it is considered  a special case of Dynkin's formula~\cite[Theorem 11.3.1]{meyn2012markov}, or as a generalization of Foster's criterion~\cite[Proposition 4.5]{bramson2008stability}. In these contexts, drift analysis is often called \emph{Lyapunov Function Method}, e.g.~\cite[Theorem 2.6.2]{menshikov2016non}. However, the hitting time is often only a side aspect in these areas.} which they stripped of its substantial technical overhead due to the fact that Hajek's focus was more on deciding whether hitting times actually exist for unbounded state spaces. He and Yao proved their theorem using (without explicit reference) the \emph{Optional Stopping Theorem} for martingales~\cite{grimmett2001probability}.
Here we give an elementary proof taken from~\cite{lengler2016drift}, since this proof gives some insight in the differences between upper and lower bounds.

\begin{theorem}[Additive Drift Theorem~\cite{HeYao:04:drift}]\label{driftthm:additive}
Let $(X_t)_{t\geq 0}$ be a sequence of non-negative random variables with a finite state space $\mathcal{S} \subseteq \R_0^+$ such that $0 \in \mathcal S$. Let $T := \inf\{t \geq 0 \mid X_t =0\}$. 
\begin{enumerate}[(a)]
\item If there exists $\delta > 0$ such that for all $s\in \mathcal S \setminus\{0\}$ and for all $t\geq 0$,
\begin{align}\label{drifteq:additivecond}
\Delta_t(s):=E[X_t-X_{t+1}\mid X_t=s]\ge \delta,
\end{align}
then
\begin{align}\label{drifteq:additiveconc}
E[T] \leq \frac{E[X_0]}{\delta}.
\end{align}
\item If there exists $\delta > 0$ such that for all $s\in \mathcal S \setminus\{0\}$ and for all $t\geq 0$,
\begin{align}\label{drifteq:additivecondlower}
\Delta_t(s):=E[X_t-X_{t+1}\mid X_t=s]\le \delta,
\end{align}
then
\begin{align}\label{drifteq:additiveconclower}
E[T] \geq \frac{E[X_0]}{\delta}.
\end{align}
\end{enumerate}
\end{theorem}
\begin{proof}
\noindent \emph{(a)} As we are only interested in the hitting time $T$ of zero we may assume without loss of generality  that $X_{T+1} = X_{T+2} = \ldots = 0$. 

We may rewrite condition~\eqref{drifteq:additivecond} as $E[X_{t+1}\mid X_t=s]\le E[X_{t}\mid X_t=s]- \delta$. Since this holds for all $s\in \mathcal S \setminus\{0\}$, and since $T>t$ if and only if $X_t >0$, we conclude
\begin{equation}\label{drifteq:additive0}
E[X_{t+1}\mid T>t]\le E[X_{t}\mid T>t]- \delta.
\end{equation}
By the law of total probability we have 
\begin{align}\label{drifteq:additive1}
E[X_t] &= \Pr[T>t] \cdot E[X_t \mid T>t] + \Pr[T\le t] \cdot \underbrace{E[X_t \mid T\leq t]}_{=0} \nonumber\\
& = \Pr[T>t] \cdot E[X_t \mid T>t].
\end{align}
Proceeding similarly for $X_{t+1}$  we obtain 
\begin{align}\label{drifteq:additive2}
E[X_{t+1}]&\;\;=\;\;\Pr[T>t] \cdot E[X_{t+1} \mid\! T>t] + \Pr[T\le t]\cdot \underbrace{E[X_{t+1} \mid T\leq t]}_{=0}\nonumber\\
& \stackrel{(\ref{drifteq:additive0})}\le    \Pr[T>t] \cdot (E[X_{t} \mid T>t] -  \delta )\nonumber\\ 
& \stackrel{(\ref{drifteq:additive1})}=   E[X_t] - \delta\cdot \Pr[T >t].
\end{align}
Since $T$ is a random variable that takes values in $\N_0$, we may write $E[T] = \sum_{t=0}^{\infty}\Pr[T>t]$. Thus 
\begin{align}\label{drifteq:additive3}
\delta\cdot E[T] 
& \stackrel{\tau\to\infty}{\longleftarrow}  \sum_{t=0}^\tau \delta\Pr[T > t] 
\; \stackrel{(\ref{drifteq:additive2})}\le\;   \sum_{t=0}^\tau (E[X_{t}] - E[X_{t+1}])
\; = \; E[X_0] - \underbrace{E[X_{\tau+1}]}_{\geq 0} \nonumber \\
& \; \leq\;  E[X_0],
\end{align}
which proves (a). 

\noindent \emph{(b)} Analogously to (a), equations~\eqref{drifteq:additive0}, \eqref{drifteq:additive1}, \eqref{drifteq:additive2}, and \eqref{drifteq:additive3} hold with reversed inequalities, except for the very last step in \eqref{drifteq:additive3}. So \eqref{drifteq:additive3} becomes:
\begin{align}\label{drifteq:additive4}
\delta\cdot E[T]  & \stackrel{\tau\to\infty}{\longleftarrow}  \sum_{t=0}^\tau \delta\Pr[T > t]  \geq E[X_0] - E[X_{\tau+1}].
\end{align}
There are only two possible cases. Either $\Pr[T > t]$, which is a non-increasing sequence, does \emph{not} converge to $0$. In this case, $E[T] = \sum_{t=0}^{\infty}\Pr[T>t] = \infty$, in which case (b) holds trivially. Or $\Pr[T > t] \to 0$, and by~\eqref{drifteq:additive1} we also have 
\begin{align}\label{drifteq:additive5}
E[X_{\tau+1}] = \underbrace{\Pr[T>t]}_{\to \; 0} \cdot \underbrace{E[X_{\tau+1} \mid T>t]}_{\leq\;  \max S \; <\;  \infty} \to 0.
\end{align}
Now (b) follows from~\eqref{drifteq:additive4} and~\eqref{drifteq:additive5}.
\end{proof}

The proof also shows what can generally go wrong for infinite search spaces. The proof of (a) goes through unmodified. For (b), inequality~\eqref{drifteq:additive4} is generally true. Moreover, it is \emph{tight} if condition~\eqref{drifteq:additivecondlower} is tight. The problem is that $E[X_{\tau+1}]$ may not go to zero. For example,\label{ex:infinitesearchspace} consider the Markov chain where $X_{t+1}$ is either $0$ or $2X_t$, both with probability $1/2$. Here $E[T]=2$, but $E[X_t-X_{t+1}]=0$ for all $t\geq 0$. In particular, condition~\eqref{drifteq:additivecondlower} is satisfied with $\delta=1$ (or any other $\delta>0$), but the conclusion of (b) does not hold. On the other hand, for the tight choice $\delta=0$, we see that we have equality in~\eqref{drifteq:additive4} since $E[X_{\tau+1}] = E[X_0]$.

Note that if the drift in Theorem~\ref{driftthm:additive} is \emph{exactly} $\delta$ in each step, then the upper and lower bounds match. In this case, Theorem~\ref{driftthm:additive} can be seen as an \emph{invariance theorem}, which states that the expected hitting time of $0$ is \emph{independent} of the exact distribution of the progress, as long as the expectation of the progress (i.e., the drift) remains fixed. In particular, if $X_0$ is an integer multiple of $\delta$, this includes the deterministic case in which $X_t$ decreases in each step by exactly $\delta$, with probability $1$. Thus a process of constant drift can not be accelerated (or slowed down) by redistributing the probability mass. We will resume this point in Section~\ref{driftsec:variable} when we discuss why other drift theorem are \emph{not} tight in general.

We conclude the section on additive drift with an application.
\begin{example}[RLS on \lo]\label{driftexample:RLSLO}
Consider \emph{Random Local Search} (RLS) on the $n$-dimensional hypercube $\{0,1\}^n$. RLS is a (1+1)-algorithm (i.e., it has population size one and generates only one offspring in each generation). The mutation operator flips exactly one bit, which is chosen uniformly at random. RLS has elitist selection, i.e., the offspring replaces the parent if and only if its fitness is at least as large as the parent's fitness. A pseudocode description is given in Algorithm~\ref{driftalgo:RLS}.
\begin{algorithm2e}\label{driftalgo:RLS}
 Choose $x^{(0)} \in \{0,1\}^n$ uniformly at random\;
\Indp
 \Indm
 \For{$t=0,1,2,\ldots$}{
 Pick $i \in \{1,\ldots,n\}$ uniformly at random, and create $y^{(t)}$ by flipping the $i$-th bit in $x^{(t)}$\; 
 \eIf{$f(y^{(t)}) \geq f(x^{(t)})$}{$x^{(t+1)} \assign y^{(t)}$\;}{$x^{(t+1)} \assign x^{(t)}$\;}
	 }
 \caption{Random Local Search (RLS) maximizing a fitness function $f:\{0,1\}^n \rightarrow \R$.}
\end{algorithm2e}

We study RLS on the \lo fitness function, which returns the number of initial one-bits before the first zero bit. Formally, 
\[
\lo(x) = \sum_{k=1}^n \prod_{i=1}^k x_i = \max\{i\in \{1,\ldots,n\} \mid \underbrace{11\ldots 1}_{i\text{ times}} \text{ is a prefix of } x\}. 
\]
The \lo problem is a classical benchmark problem for evolutionary algorithms, and RLS on \lo has been studied in much greater detail than we can present here, with methods and results that go far beyond drift analysis~\cite{doerr2016impact, ladret2005asymptotic}.

\noindent {\bf Naive potential.} As potential we choose in a first step $X_t := n- f(x^{(t)})$, the distance in fitness from the optimum. The state space is $\mathcal S = \{0,\ldots,n\}$. We need to compute the drift $\Delta_t(s) := E[X_{t} - X_{t+1} \mid X_t = s]$ for every state $s \in \mathcal S\setminus \{0\}$, so we fix such an $s$. For convenience, we write $k := n-s \in \{0,\ldots,n-1\}$ for the fitness in this case. Note that $X_t = s$ implies that the first $k$ bits of $x^{(t)}$ are one-bits, but the $k+1$-st bit is a zero-bit. Obviously, the potential changes if and only if we flip the $k+1$-st bit, so let us denote this event by $\mathcal E$. Since the flipped bit is chosen uniformly, we have $\Pr[\mathcal E] = 1/n$. Hence the drift is
\begin{align}\label{drifteq:RLSLO1}
\Delta_t(s) & = \Pr[\mathcal E] \cdot \underbrace{E[X_{t} - X_{t+1} \mid X_t = s \text{ and } \mathcal E]}_{=:E(s)}  = \tfrac 1n \cdot E(s).
\end{align}
So it remains to bound the conditional expectation $E(s)$. Such conditional expectations occur quite frequently when a drift is computed. Assume that $X_t = s$ (i.e., $f(x^{(t)})= k = n-s$), and that $\mathcal E$ occurs. Obviously, $E(s) \geq 1$, since we improve at least the $k+1$-st bit. On the other hand, we improve the fitness by at least $2$ if and only if the $k+2$-nd bit happens to be a one-bit. Note that since the algorithm is elitist and has fitness $f(x_t)=k$, the $k+2$-nd bit has had no influence on the fitness of previous search points. Therefore, by symmetry, it has probability $1/2$ to be a one-bit\footnote{Note that such an argument would not be true if we would condition on one particular history of the algorithm, cf. the discussion on filtrations in Section~\ref{driftsec:situation}.} and we obtain 
\[
\Pr[X_t-X_{t+1} \geq 2 \mid X_t = s \text{ and } \mathcal E] = Pr[x^{(t)}_{k+2} = 1 \mid X_t = s \text{ and } \mathcal E] = 1/2.
\] 
Analogously, $X_t-X_{t+1} \geq i$ if and only if the bits with indices $k+2, \ldots, k+i$ are all one-bits, which happens with probability $2^{-i+1}$. Since $X_t-X_{t+1}$ is an integer non-negative random variable, we may sandwich
\begin{align}\label{drifteq:RLSLO2}
1\leq E(s) & = \sum_{i=1}^{s} \Pr[X_t-X_{t+1} \geq i \mid X_t = s \text{ and } \mathcal E] \nonumber \\
& = 1+\sum_{i=2}^{s} 2^{-i+1} < 1+\sum_{i=1}^{\infty} 2^{-i} =2
\end{align}
Hence, by~\eqref{drifteq:RLSLO1},
\begin{align}
\tfrac{1}{n} \leq \Delta_t(k)  \leq \tfrac{2}{n},
\end{align}
and Theorem~\ref{driftthm:additive} implies that
\begin{align}\label{drifteq:RLSLO3}
\frac{n}{2} E[X_0] \leq E[T] \leq n E[X_0].
\end{align}
To estimate $E[X_0] = n - E[f(x^{(0)})]$, we observe that $f(x^{(0)}) \geq i$ happens if only if the first $i$ bits are all one-bits, which happens with probability $2^{-i}$. Hence, a similar calculation as before shows 
\begin{align}\label{drifteq:RLSLO4}
E[f(x^{(0)})] = \sum_{i=1}^n \Pr[f(x^{(0)}) \geq i] = \sum_{i=1}^n 2^{-i} = 1-2^{-n} \in [0,1],
\end{align}
and thus $n-1 \leq E[X_0] \leq n$. Therefore, by \eqref{drifteq:RLSLO3} we get $\tfrac{(n-1)n}{2}\leq E[T]\leq n^2$, and thus $E[T] = \Theta(n^2)$.

\noindent {\bf Translated potential.} The analysis so far gives the asymptotics $E[T]$, but it is not tight up to constant factors. The problem is, as \eqref{drifteq:RLSLO2} shows, that the inequality $E(k) \geq 1$ is rather coarse except for the few exceptional cases where $k$ is almost $n$. In fact, in the border case $k=n-1$ we have equality, $E(k)=1$. Hence, we do not have a perfectly constant drift, which is a reason for the discrepancy between upper and lower bound. Such border effects can often be remedied by \emph{translating the potential function}. In this case, we consider
\begin{align}\label{drifteq:RLSLOY}
Y_t & := \begin{cases} X_t +1 , & \text{if $X_t \geq 1;$ }\\  0, & \text{otherwise.}\end{cases}
\end{align}
The effect is that the drift increases when there is a substantial chance to reach $0$ in the next step. In our case, we get an additional term for $i=n-k+1$ in~\eqref{drifteq:RLSLO2}, which equals the term for $i=n-k$. Intuitively, the term for $i=n-k$ counts double since in this case the potential drops from $2$ to $0$, rather than from $1$ to $0$. Consequently, we get for the potential $Y_t = s+1$, which corresponds as before to fitness $f(x^{(t)})=k=n-s$:
\begin{align}\label{drifteq:RLSLO5}
E[Y_{t} - Y_{t+1} \mid Y_t = s+1 \text{ and } \mathcal E] & = \sum_{i=1}^{n-k} \Pr[Y_t-Y_{t+1} \geq i \mid Y_t = s+1 \text{ and } \mathcal E] \nonumber \\
& = 1+\sum_{i=2}^{n-k} 2^{-i+1} + 2^{n-k+1} =2.
\end{align}
Hence, the drift with respect to $Y_t$ is exactly $2/n$, and Theorem~\ref{driftthm:additive} gives a tight result:
\begin{align}\label{drifteq:RLSLO6}
E[T] = \frac{n}{2} E[Y_0].
\end{align}
From~\eqref{drifteq:RLSLOY} it is easy to compute $E[Y_0]$ exactly as
\begin{align*}
E[Y_0] = n- E[f(x^{(0)})]+1\cdot \Pr[Y_0 >0] \stackrel{\eqref{drifteq:RLSLO4}}{=} n- (1-2^{-n}) + 1-2^{-n} = n.
\end{align*}
Together with~\eqref{drifteq:RLSLO6}, the Additive Drift Theorem~\ref{driftthm:additive} now implies $E[T] = n^2/2$.
\end{example}

The previous example illustrates how important it is for Theorem~\ref{driftthm:additive} that the drift be as uniform as possible, to get matching upper and lower bounds. The example also shows that rescaling of the potential function may be a way to smoothen out inhomogeneities. Following this approach systematically leads to the variable drift theorem that we will discuss in the next section.

\subsection{Variable Drift}\label{driftsec:variable}

The Additive Drift Theorem is useful because it is tight, but it requires us to find a potential function that has constant drift. Is this even always possible? The perhaps surprising answer is ``Yes'', as we will discuss in Section~\ref{driftsec:potential}. Unfortunately, it can be rather hard to find a good potential. However, there are helpful tools. Even if we start with a potential functions in the ``wrong'' scaling, Mitavskiy, Rowe, and Cannings~\cite{Mit-Row-Can:j:09}, and Johannsen in his PhD thesis~\cite{Joh:th:10} developed a theorem which \emph{automatically} rescales the drift in the right way. A similar result has been obtained independently (and earlier) by Baritompa and Steel~\cite{baritompa1996bounds}
\begin{theorem}[{Variable Drift Theorem \cite{Joh:th:10,DBLP:conf/gecco/RoweS12}}]
\label{driftthm:variable}
Let $(X_t)_{t\geq 0}$ be a sequence of non-negative random variables with a finite state space $\mathcal{S} \subseteq \R_0^+$ such that $0 \in \mathcal S$. Let $s_{\text{min}} := \min(\mathcal{S} \setminus \{0\})$, let $T := \inf\{t \geq 0 \mid X_t =0\}$, and for $t \geq 0$ and $s \in \mathcal S$ let $\Delta_t(s) := E[X_{t} - X_{t+1} \mid X_t = s]$. If there is an increasing function\footnote{Some formulations in the literature require $h$ to be integrable. However, since we assume $\mathcal{S}$ to be finite, the interval $[s_{\min},X_0]$ is a compact interval, on which every monotone function is integrable.} $h: \R^+ \rightarrow \R^+$ such that for all $s \in \mathcal{S}\setminus\{0\}$ and all $t\geq 0$,
\begin{align}\label{drifteq:variabledrift1}
\Delta_t(s) \geq h(s),
\end{align}
then
\begin{align}\label{drifteq:variabledrift2}
E[T] \leq \frac{s_{\text{min}}}{h(s_{\text{min}})} + E\left[\int_{s_{\text{min}}}^{X_0} \frac{1}{h(\sigma)} d\sigma\right],
\end{align}
where the expectation on the latter term is over the random choice of $X_0$.
\end{theorem}

We remark that the condition that $h$ be increasing is usually satisfied, since progress typically becomes harder as the algorithm approaches an optimum. We will see in the proof why the condition is necessary, and an example showing that it is necessary can be found in~\cite{kotzing2018first}. However, variants of the theorem for non-decreasing drift function do exist~\cite{doerr2012ants,feldmann2013optimizing}.

We present a proof of the Variable Drift Theorem, for two reasons. Firstly, the theorem is so central that it deserves to come with a proof. Secondly, we will gain valuable insights from the proof. In particular, it will enable us to understand when the upper bound on $E[T]$ is tight, and realise when the upper bound may be misleading. A reader who is completely new to drift analysis may first skip ahead to some examples, and return to the proof when we discuss tightness of the Variable Drift Theorem.

\begin{proof}[Proof of Theorem~\ref{driftthm:variable}, adapted from~\cite{Joh:th:10}]
The main insight of the proof lies in an appropriate rescaling of $X_t$ by the function
\begin{align}\label{drifteq:proofvariable0}
g(s) := 
\begin{cases}\frac{s_{\text{min}}}{h(s_{\text{min}})} + \int_{s_{\text{min}}}^{s}\frac{1}{h(\sigma)}d\sigma, & s \geq s_{\text{min}},\\
\frac{s}{h(s_{\text{min}})}, & 0 \leq s \leq s_{\text{min}}.
\end{cases}
\end{align}
The integral is well-defined since $h$ is increasing. Note that $g$ is strictly increasing. We claim that for all $s \in \mathcal{S}\setminus \{0\}$ and all $r \geq 0$,
\begin{align}\label{drifteq:proofvariable1}
g(s)-g(r) \geq \frac{s-r}{h(s)}.
\end{align}
To prove the claim, we distinguish three cases. First assume $s \geq r \geq s_{\text{min}}$. Then
\begin{align}\label{drifteq:proofvariable2}
g(s)-g(r) = \int_{r}^s \frac{1}{h(\sigma)} d\sigma \geq \int_{r}^s \frac{1}{h(s)} d\sigma = \frac{s-r}{h(s)}.
\end{align}
Similarly, if $r \geq s \geq s_{\text{min}}$, then
\begin{align}\label{drifteq:proofvariable3}
g(r)-g(s) = \int_{s}^r \frac{1}{h(\sigma)} d\sigma \leq \int_{s}^r \frac{1}{h(s)} d\sigma = \frac{r-s}{h(s)},
\end{align}
and multiplication with $-1$ yields the claim. The only remaining case is $s \geq s_{\text{min}} > r\geq 0$ (since we assumed $s\in  \mathcal{S}\setminus \{0\}$), and in this case,
\begin{align}\label{drifteq:proofvariable4}
g(s)-g(r) & = \frac{s_{\text{min}}}{h(s_{\text{min}})} + \int_{s_{\text{min}}}^s \frac{1}{h(\sigma)} d\sigma - \frac{r}{h(s_{\text{min}})} \geq  \frac{s_{\text{min}}-r}{h(s_{\text{min}})} +\frac{s-s_{\text{min}}}{h(s)} \nonumber \\ 
& \geq \frac{s-r}{h(s)}.
\end{align}
Now let us consider the rescaled random variable $Y_t := g(X_t)$. This random variable takes values of the form $g(s)$, where $s\in S$. For all $s \in \mathcal{S} \setminus \{0\}$,
\begin{align}\label{drifteq:proofvariable5}
E[Y_t-Y_{t+1}\mid Y_t = g(s)] & = E[g(X_t)-g(X_{t+1})\mid g(X_t) = g(s)] \nonumber \\
&\stackrel{\eqref{drifteq:proofvariable1}}{\geq} E\left[\frac{X_t-X_{t+1}}{h(X_t)}\; \middle|\; X_t = s\right] = \frac{\Delta_t(s)}{h(s)} \stackrel{\eqref{drifteq:variabledrift1}}{\geq} 1.
\end{align}
Hence $Y_t$ has at least a constant drift. The theorem follows by applying the Additive Drift Theorem~\ref{driftthm:additive} to $Y_t$.
\end{proof}

\begin{example}[Coupon Collector, RLS on \om]\label{driftexample:RLSOM}
The most classical example for variable drift is the \emph{Coupon Collector Process (CCP)}: there are $n$ types of coupons, and a collector wants to have at least one coupon of each type. However, the coupons are sold in opaque wrappings, so she cannot see the type of a coupon before buying it. If each type occurs with the same frequency $1/n$, how many coupons does she need to buy before she has every type at least once?

The CCP and its variants appear in various contexts within the study of EAs. The most basic example is the runtime of RLS (Algorithm \ref{driftalgo:RLS} on page \pageref{driftalgo:RLS}) for maximising the \om fitness function, which counts the number of one-bits in a bitstring. Formally, for $x \in \{0,1\}^n$,
\begin{align}\label{drifteq:defOneMax}
\om(x) = \sum_{i=1}^n x_i. 
\end{align}
The one-bits correspond to the coupons in the CCP that the collector has already obtained. Since RLS flips in each round exactly one bit, and a one-bit stays a one-bit forever, a round of RLS corresponds exactly to the purchase of a coupon. Thus the number of rounds of RLS on \om is equivalent to the number of purchases in the CCP.\footnote{except for the initial conditions: for the CCP, the collector usually starts with no coupons, while RLS starts with a random bitstring and thus with a random initial number of ones/coupons.}

To analyse the CCP, we let $X_t$ be the number of missing coupons after $t$ purchases, and as usual we denote by $T$ the hitting time of $0$. Then for $X_t=s$ the probability to obtain a new type with the next purchase is $s/n$. In this case $X_t$ decreases by one, so $X_t$ has a drift of $\Delta_t(s) = s/n$. The minimal positive value of $X_t$ is $s_{\text{min}} = 1$. Hence, the Variable Drift Theorem with function $h(s)=s/n$ gives the upper bound
\begin{align}\label{drifteq:exampleCCP}
E[T] \leq \frac{1}{h(1)} + E\left[\int_{1}^{X_0} \frac{n}{\sigma} d\sigma\right] = n(1+E[\ln(X_0)]) \leq n\ln n + n.
\end{align}
\end{example}\bigskip

The drift in Example~\ref{driftexample:RLSOM} was \emph{multiplicative}, i.e., $\Delta_t(s)$ was proportional to $s$. This is by far the most important special case of the Variable Drift Theorem, important enough that in Section~\ref{driftsec:multiplicative} we will provide it with a theorem of its own, the Multiplicative Drift Theorem. A reader who is eager to see some more cute examples of a similar type is invited to peek ahead.

The upper bound in Example~\ref{driftexample:RLSOM} is remarkably tight. The expected runtime is indeed $E[T] = n \ln n + \Theta(n)$, both for CCP~\cite{MotwaniR97} and for RLS on \om~\cite{doerr2016impact}. We will discuss in the next section when we can expect the bounds from the Variable Drift Theorem to be tight, and see situations in which they are rather inaccurate. 
Before that, we give a more serious example coming from applications. 

\begin{example}[Genetic Programming]\label{driftexample:GP}
Genetic Programming (GP) uses evolutionary principles to automatically generate programs which match some desired input-output schemes. The programs are typically represented as syntax trees~\cite{koza1992genetic}, where the leaves correspond to variables $x_1,\ldots,x_n$, and the inner nodes correspond to operators like \AND, \OR, or \NOT. Here we restrict ourselves to the Boolean domain, for simplicity. Then each syntax tree $\tau$ represents a Boolean term, and thus defines a pseudo-Boolean function $f_\tau : \{0,1\}^n \to \{0,1\}$. Doerr, Lissovoi and Oliveto~\cite{doerr2018evolution} studied the problem of learning the \AND function $\AND(x_1,\ldots,x_n) = x_1 \wedge \ldots \wedge x_n$, if the inner nodes may either be \AND or \OR. To turn it into an optimisation problem, we assign to each syntax tree $\tau$ the number $F(\tau)$ of inputs $x = (x_1,\ldots,x_n) \in \{0,1\}^n$ for which $f_\tau(x_1,\ldots,x_n) \neq \AND(x_1,\ldots,X_n)$. So the goal is to reduce the potential $F$ to zero. The search procedures considered in~\cite{doerr2018evolution} uses a mutation operator which adds, substitutes, or deletes nodes, or which deletes whole subtrees of the current syntax tree. The actual algorithm is rather complicated, and we refer the reader to~\cite{doerr2018evolution} for more details.

We call $X_t := F(\tau_t)$, where $\tau_t$ is the syntax tree after $t$ steps. Then the authors could show that $X_t$ has the following drift. 
\begin{align}\label{drifteq:exampleGP1}
E[X_{t} - X_{t+1} \mid X_t = s] \geq h(s) := \begin{cases}\frac{\delta s\ln s}{\ln n}, & \text{ if $s \geq n$,} \\ \delta s, & \text{ if $s < n$,} \end{cases}
\end{align}
where $\delta = \Theta(1/n^2)$ depends on the number of variables, but is independent of $s$. Note that $h$ is increasing and that $X_0 \leq 2^n$. Therefore, the Variable Drift Theorem immediately gives the following upper bound on the expected optimisation time $T$.
\begin{align}\label{drifteq:exampleGP2}
E[T] \leq \frac{1}{h(1)} + \int_{1}^{2^n} \frac{1}{h(\sigma)} d\sigma = \frac{1}{h(1)} + \int_{1}^{n} \frac{1}{\delta \sigma} d\sigma + \int_{n}^{2^n} \frac{\ln n}{\delta \sigma \ln \sigma} d\sigma.
\end{align}
To compute the integral, we note that the inverse derivative of $1/\sigma$ is $\ln \sigma$, and the inverse derivative of $1/(\sigma \ln \sigma)$ is $\ln \ln \sigma$. Hence,
\begin{align}\label{drifteq:exampleGP3}
E[T] \leq  \frac{1}{\delta}+ \frac{\ln n}{\delta} + \frac{\ln n}{\delta}(\ln \ln 2^n - \ln \ln n) = O(\log^2 n/\delta) = O(n^2\log^2 n).
\end{align}
So once we have found the drift as in~\eqref{drifteq:exampleGP1}, the drift theorems make it an easy task to compute the expected runtime. Of course, the main contribution of the authors is to actually compute the drift.
\end{example}

\paragraph{Tightness of the Variable Drift Theorem.}\label{driftsec:tightness}

In general, the bound in the Variable Drift Theorem~\ref{driftthm:variable} does not need to be tight, even if we assume that $h(s)$ is a tight lower bound for the drift (i.e., if~\eqref{drifteq:variabledrift1} is an equality). However, in many situations the bound \emph{is} tight, especially if the potential $X_t$ does not jump around too much. Let us unravel the proof of Theorem~\ref{driftthm:variable} to understand this phenomenom better. 

We first note that the proof is a reduction to the Additive Drift Theorem, which \emph{is} tight (cf.~the discussion after Theorem~\ref{driftthm:additive}). So the only possible problem is the estimate~\eqref{drifteq:proofvariable5} on the drift. This estimate may not be tight if~\eqref{drifteq:proofvariable1}, the inequality $g(s)-g(r) \geq \frac{s-r}{h(s)}$, is too coarse. Note that for estimating the drift, we use~\eqref{drifteq:proofvariable1} specifically for $s= X_t$ and $r = X_{t+1}$. These are not arbitrary values; for example, for RLS on \om, they differ by at most one. We have proved~\eqref{drifteq:proofvariable1} by case distinction, so let us inspect one of the cases for illustration. For convenience, we restate the argument for $s>r> s_{\text{min}}$:
\begin{align}\tag{\ref{drifteq:proofvariable2}}
g(s)-g(r) = \int_{r}^s \frac{1}{h(\sigma)} d\sigma \geq \int_{r}^s \frac{1}{h(s)} d\sigma = \frac{s-r}{h(s)}.
\end{align}
The crucial step is too use $1/h(\sigma) \geq 1/h(s)$ for the range $r\leq \sigma \leq s$. In general, this may be a bad estimate. However, if $s= X_t$ and $r = X_{t+1}$ are close to each other then $\sigma$ runs through a small range, and $1/h(\sigma)$ may not vary too much. For example, $s$ and $r$ differ at most by one for RLS on \om, and the function $1/h(\sigma) = n/\sigma$ does not vary much in such a small range, especially if $r$ and $s$ are large. We will see in Section~\ref{driftsec:lower} that large jumps are still tolerable if they occur with sufficiently small probability. The following artificial example from~\cite{giessen2017optimal} illustrates how large jumps can lead to bad upper bounds. The idea of the construction is similar to the initial example from page~\pageref{ex:infinitesearchspace}.

\begin{example}[RLS with shortcuts]\label{driftexample:RLSshortcuts}
Consider a $(1+1)$-algorithm that in each step creates the optimum with probability $1/n$, and with probability $1-1/n$ it does an RLS step as in Algorithm~\ref{driftalgo:RLS}. For minimising \om, we may naively try the fitness as potential, $X_t := \om(x^{(t)})$. For $X_t = s >0$, there is a probability of $1/n$ to jump directly to the optimum, thus decreasing the potential by $s$. On the other hand, there is a probability of $(1-1/n) \cdot i/n$ to decrease the potential by $1$ with a normal RLS step. Together, the drift is
\begin{align}\label{drifteq:RLSshortcuts1}
\Delta_t(s) = h(s) := \frac{1}{n}\cdot s + \left(1-\frac{1}{n}\right)\frac{s}{n} = \frac{2s}{n} - \frac{s}{n^2} = (1\pm o(1))\frac{2s}{n}.
\end{align}
Thus, the Variable Drift Theorem~\ref{driftthm:variable} yields
\begin{align}\label{drifteq:RLSshortcuts2}
E[T] \leq \frac{1}{h(1)} + E\left[\int_{1}^{X_0} (1\pm o(1))\frac{n}{2\sigma}d\sigma\right] = \Theta(n\log n).
\end{align}
However, since in each step we have probability at least $1/n$ to jump directly to the optimum, the expected runtime is at most $E[T] \leq n$, so~\eqref{drifteq:RLSshortcuts2} is not tight. The problem can be understood by inspecting the transformed variable $Y_t := g(X_t)$ from the proof of the Variable Drift Theorem, equation~\eqref{drifteq:proofvariable0}. For simplicity we ignoring the factor $(1+o(1))$ in~\eqref{drifteq:RLSshortcuts1}, and obtain
\begin{align}\label{drifteq:RLSshortcuts3}
Y_t := \begin{cases} \frac{n}{2}(1+\ln X_t) & \text{ if } X_t \geq 1,\\
0 & \text{ if } X_t =0 \end{cases}.
\end{align}
Computing the drift of $Y_t$ directly, we obtain for $X_t = s$, i.e, for $Y_t = \frac{n}{2}(1+\ln s)$.
\begin{align}\label{drifteq:RLSshortcuts4}
E[Y_t-Y_{t+1} \mid X_t = s] &=  \frac{1}{n}\cdot\frac{n}{2}(1+\ln s) + \left(1-\frac{1}{n}\right)\frac{s}{n}\cdot\frac{n}{2}\left(\ln s -\ln(s-1)\right) \nonumber\\
& = \frac{\ln s}{2} \pm O(1).
\end{align}
Thus, we do not have constant drift in the scaled potential. However, in the proof of the Variable Drift Theorem~\eqref{driftthm:variable}, we bound the drift by 1 (see~\eqref{drifteq:proofvariable5}), which is the reason for the additional $\log n$ factor. 
\end{example}

Fortunately, it is quite common that there are no large jumps of fitness values. Mutation-based evolutionary algorithms tend to make small steps, and other nature-based search heuristics like ant-colony optimisation or estimation of distribution algorithms tend to make rather small updates on reasonable functions. However, note that this is not necessarily true for crossover operations. Also, depending on the fitness function a small (genotypical) change may cause a large (phenotypical) jump in the fitness, as the next example shows.

\begin{example}[RLS on \BinVal]\label{driftexample:RLSBinVal}

We consider RLS (Algorithm \ref{driftalgo:RLS} on page \pageref{driftalgo:RLS}) for minimising the \BinVal function given by
\begin{align}\label{drifteq:defBinVal}
\BinVal(x) = \sum_{i=1}^n 2^{n-i} x_i. 
\end{align}
If we choose the potential $X_t := \BinVal(x^{(t)})$ identical to the fitness, then we observe that each one-bit has probability $1/n$ to be flipped. If the $i$-th bit is flipped from one to zero, this reduces the potential by $2^i$. Hence, at search point $x$ with potential $s:= \BinVal(x)$ the drift is
\begin{align}\label{drifteq:binval1}
E[X_t-X_{t+1} \mid x^{(t)} = x] = \sum_{1 \leq i \leq n,\ x_i^{(t)} = 1} \frac{1}{n}\cdot 2^{n-i} = \frac{1}{n} \sum_{i=1}^n 2^{n-i} x_i = \frac{s}{n}. 
\end{align}
In particular, since the latter term only depends on $s$, we can write
\begin{align}\label{drifteq:binval2}
E[X_t-X_{t+1} \mid X_t = s] = \frac{s}{n}. 
\end{align}
Therefore we are in the situation to apply the Variable Drift Theorem~\ref{driftthm:variable} with $h(s) = s/n$ and $s_{\text{min}} = 1$, and obtain
\begin{align}\label{drifteq:binval3}
E[T] \leq \frac{1}{1/n} + E\left[\int_{1}^{X_0} \frac{n}{\sigma} d\sigma\right] = n + n\cdot E[\ln X_0] = \Theta(n^2),
\end{align}
where the last equality follows since $X_0 \leq 2^{n+1}$, and since with probability at least $1/2$ the first bit in $X_0$ is a one-bit, which implies $E[X_0]\geq 2^{n-1}$. 

However, the bound~\eqref{drifteq:binval3} is far from tight. In fact, if we use the \emph{OneMax potential} $\om(x):= \sum_{i=1}^n x_i$, then the drift with respect to $\om$ is still $\Delta_t^{\om}(s) = s/n$, which leads to a runtime bound of $E[T] \leq n+n\cdot E[\om(x^{(0)})] \leq n\ln n +n$.\footnote{Alternatively, we could observe that RLS behaves exactly the same on \BinVal and on \om, so the runtimes are the same.} 

The reason why~\eqref{drifteq:binval3} is not tight is that there may be some very large jumps in the potential (cf. the discussion before this example). For example, consider  the situation when only a single one-bit is left. RLS operates symmetrically on \BinVal, so this one-bit is at a random position.\footnote{Note that this is specific to RLS, which uses only one-bit flips. An algorithm which flips two or more bits per step would not operate symmetrically since it would trade a one-bit of large weight for a zero-bit of low weight, but not vice versa.} In particular, with probability at least $1/2$, the bit is in the first half, and thus $X_{t} \geq 2^{n/2}$. Therefore, in equation~\eqref{drifteq:proofvariable4} we estimate $h(\sigma) \leq h(s)$ for $\sigma$ which ranges at least between $s_{\text{min}}=1$ and $2^{n/2}$. Thus the estimate is off by an exponential factor. Consequently, the rescaled potential $Y_t = g(X_t) = n(1+ \ln X_t)$ does not have constant drift. While the drift is always at least $1$ by equation~\eqref{drifteq:proofvariable5}, if there is only a single one-bit left in the first half of the string, the rescaled potential decreases with probability $1/n$ from $Y_t \geq n(1+\ln 2^{n/2})= \Omega(n^2)$ to $0$. Hence, the drift of $Y_t$ in this situation is $1/n \cdot \Omega(n^2) = \Omega(n)$, causing the runtime bound to be almost a factor $n$ too large.
\end{example}

\paragraph{When Rescaling Beats the Variable Drift Theorem}

We have seen an example which illustrates why the Variable Drift Theorem does not always give tight results. Unfortunately, a common reason is that the potential does not represent very well the progress the algorithm has made, in which case a truly new insight is needed. However, sometimes the problem can be solved by directly considering the rescaled potential. We illustrate this by an artificial example taken from~\cite{lengler2016drift}. 

\begin{example}[Random Decline]\label{driftexample:randomdecline}
Let $a>0$ be a constant, let $n\in \N^+$, and consider the following Markov chain on $\mathcal{S}=\{0,\ldots,N\}$, where $N$ is a sufficiently large integer compared to $n$. For this exposition we will assume that $N$ is so large that the process never hits the right border. We start with $X_0 = n$, and for each $t\geq 0$ we draw $X_{t+1}$ uniformly at random from $\{0,1,2,\ldots,\min\{\lfloor aX_t\rfloor,N\}\}$. 

If $a<2$, then for $\mathcal{S} \in S \setminus \{0\}$ and all $t\geq 0$ we have a drift of
\begin{align}\label{drifteq:randomdecline1}
\Delta_t(s) \geq s-\frac{a}{2}s = \frac{2-a}{2}\cdot s.
\end{align}
Therefore, by the Variable Drift Theorem~\ref{driftthm:variable}, $E[T] = O(\log n)$. However, the theorem does not make any statement for $a \geq 2$.\footnote{Worse: the statement could be applied for non-constant $a$ like $a = 2(1-1/n)$, and would lead to the misleading bound $E[T] = O(n \log n)$.} However, let us inspect the rescaled potential $Y_t := 1+\ln(X_t)$. We only give an estimate, the full calculation including error terms can be found in~\cite{lengler2016drift}. For every $s \in \mathcal{S}\setminus\{0\}$ that is smaller than $N/a$:
\begin{align}\label{drifteq:randomdecline2}
E[Y_t-Y_{t+1}\mid Y_t = 1+\ln s] & =  1+\ln (s) - \frac{1}{\lfloor as +1\rfloor}\sum_{k=1}^{\lfloor as\rfloor}(1+\ln k) \nonumber\\
& \approx \ln(s) - \frac{1}{as}\left( \int_1^{as}\ln \sigma\; d\sigma \right)   \nonumber\\
& =  \ln(x) - \frac{1}{as}[\sigma\ln(\sigma) - \sigma]_{\sigma=1}^{as}  \nonumber\\
& \approx  \ln(s) - (\ln(as) - 1) =  1- \ln a.
\end{align}
Thus we see that if $a< e = 2.71\ldots$ is a constant, then the drift of $Y_t$ is also constant. Hence, by the Additive Drift Theorem~\ref{driftthm:additive} we get $E[T] = O(E[Y_0]) = O(\ln n)$. So the analysis of the rescaled random variable applies to a wider range than the Variable Drift Theorem~\ref{driftthm:variable}. In fact, the condition $a<e$ is tight for logarithmic runtime, since for $a \geq e$ the expected runtime is $\omega(\ln n)$~\cite{lengler2016drift}.

We have seen that once we try out the rescaling $Y_t = 1+\ln(X_t)$, the rest is very simple and mostly calculations. We will discuss in Section~\ref{driftsec:potential} how to see that this particular rescaling is worth trying.
\end{example}

\paragraph{Further applications of the Variable Drift Theorem}

We conclude the section with some more applications of the Variable Drift Theorem. They illustrate that even if the drift is a highly complicated function, the variable drift theorem gives us an explicit expression for the expected runtime, which we can evaluate by elementary calculus. An impatient reader is free to skip this section.

\begin{example}[\oplea on \om]\label{driftexample:1lambdaOM}
In 2017, Gie\ss en and Witt~\cite{giessen2017interplay} analysed the \oplea (Algorithm~\ref{driftalgo:oplea}) for minimising the \om function, cf. equation~\eqref{drifteq:defOneMax}.
\begin{algorithm2e}\label{driftalgo:oplea}
 Choose $x^{(0)} \in \{0,1\}^n$ uniformly at random\;
\Indp
 \Indm
 \For{$t=0,1,2,\ldots$}{
 \For{$i=1,\ldots,\lambda$}{
 Create $y^{(t,i)}$ by flipping each bit of $x^{(t)}$ independently with probability $c/n$\; 
 $y^{(t)} \assign \text{argmin}\{f(y^{(t,i)})\}$ (breaking ties randomly)\; 
 \eIf{$f(y^{(t)}) \leq f(x^{(t)})$}{$x^{(t+1)} \assign y^{(t)}$\;}{$x^{(t+1)} \assign x^{(t)}$\;}
	 }
	 }
 \caption{The \oplea with offspring population size $\lambda$ and mutation rate $c/n$, minimising a fitness function $f:\{0,1\}^n \rightarrow \R$.}
\end{algorithm2e}

The potential was identical with the fitness, $X_t = \om(x^{(t)})$. To bound the drift $\Delta_t(s)$, the authors used order statistics of the binomial distribution. They could show that $\Delta_t(s) \geq h(s)$, where\footnote{for the case $\lambda = \omega(1)$. The other case $\lambda = O(1)$ is similar.}
\begin{align}\label{drifteq:example1lambdaOM1}
h(s) := \begin{cases}
(1-o(1))\tfrac{\ln \lambda}{\ln{\ln \lambda}} & \text{ if } s \geq \frac{n}{(\ln \lambda)^{1/(\ln \ln \ln \lambda)}},\\
(1/2-o(1))e^{-c}\frac{\ln \lambda}{\ln \ln \lambda} & \text{ if } s \geq \frac{n}{\ln \lambda},\\
(1-o(1))e^{-c}\min\{c,1\}/2 & \text{ if } s \geq \frac{n}{\lambda},\\
(1-o(1))e^{-c}\frac{c}{\sqrt{\ln n}} & \text{ if } s \geq \frac{n}{\lambda\sqrt{\ln n}},\\
(1-o(1))ce^{-c}\lambda\frac{s}{n} & \text{ if } s < \frac{n}{\lambda\sqrt{\ln n}}.
\end{cases}
\end{align}
Obviously, computing the drift is non-trivial, and the major contribution of the paper. Despite the complexity of the formula, once we know it we can easily obtain a runtime bound by the variable drift theorem:
\begin{align}\label{drifteq:example1lambdaOM2}
E[T] \leq \frac{1}{h(1)} + E\left[\int_{1}^{X_{\max}} \frac{1}{h(\sigma)} d\sigma\right].
\end{align}
The integral can now be computed by splitting it into six ranges, and evaluating it with elementary calculus. Actually, $h(\sigma)$ is constant for all ranges except for the last one, which gives one of the leading terms:
\begin{align}\label{drifteq:example1lambdaOM3}
\int_{1}^{n/(\lambda\sqrt{\ln n})} (1+o(1))\frac{e^cn}{c\lambda\sigma} d\sigma = (1+o(1))\frac{e^cn \ln((n/(\lambda\sqrt{\ln n}))}{c\lambda}.
\end{align}
Proceeding like this for all six ranges, the authors obtain the final result
\begin{align}\label{drifteq:example1lambdaOM4}
E[T] \leq (1+o(1))\left(\frac{e^c}{c}\cdot \frac{n \ln n}{\lambda} + \frac12 \cdot \frac{n \ln \ln \lambda}{\ln \lambda}\right).
\end{align}
The authors also prove a matching lower bound by the techniques discussed in Section~\ref{driftsec:lower}

\end{example}

\begin{example}[Island Model on \om]\label{driftexample:island}
Doerr, Fischbeck, Frahnow, Friedrich, K{\"o}tzing, and Schirneck \cite{doerr2017island} studied island models in various topologies. For the complete graph as migration topology, the algorithm consists of $\lambda$ independent \ooeas, except that every $\tau$ rounds all individuals are updated by the current best search point, see Algorithm~\ref{driftalgo:island}.
\begin{algorithm2e}\label{driftalgo:island}
 Choose $x^{(0,1)},\ldots,x^{(0,\lambda)}  \in \{0,1\}^n$ uniformly at random\;
\Indp
 \Indm
 \For{$t=0,1,2,\ldots$}{
 \For{$i=1,\ldots,\lambda$}{
 Create $y^{(t,i)}$ by flipping each bit of $x^{(t,i)}$ independently with probability $1/n$\; 
 \eIf{$f(y^{(t,i)}) \leq f(x^{(t,i)})$}{$x^{(t+1,i)} \assign y^{(t,i)}$\;}{$x^{(t+1,i)} \assign x^{(t,i)}$\;}
	 }
 \If{$(t+1 \bmod \tau) = 0$}{
 \For{$i=1,\ldots,\lambda$}{
 $y \assign \text{argmin}\{f(y^{(t+1,i)})\}$ (breaking ties randomly)\; 
$x^{(t+1,i)} \assign y$\;
	}
	}
 	 }
 \caption{Island model on $\lambda$ islands and migration interval $\tau$ for minimising $f: \{0,1\}^n \to \R$.}
\end{algorithm2e}

For minimising the \om function, the most interesting phase\footnote{for some parameter regimes} turns out to be the phase when the current best search point has fitness in some interval $[s_0,s_1]$, where $s_0 = \min\{n,n\ln \lambda/(2\tau)\}$ and $s_1 = n/(\tau \ln \lambda)$. The authors define $X_t$ to be the fitness after $t$ migrations, i.e., $X_t = \om(x^{(t\tau,i)})$ holds for every $1\leq i\leq \lambda$. To identify the end of the phase, we truncate $X_t$, i.e., we define $X_t := 0$ if $\om(x^{(t\tau,i)}) <s_0$. Note that the minimal non-zero value of $X_t$ is thus $s_{\text{min}} = s_0$. The drift of $X_t$ for all $t\geq 0$ and all $s \in [s_0,s_1]$ turns out to be
\begin{align}\label{drifteq:island1}
\Delta_t(s) \geq h(s) := \frac{c \ln \lambda}{\ln(n\ln \lambda /(\tau s))}.
\end{align}
for some constant $c>0$. Note that the function $h(s)$ is increasing. Thus, by the Variable Drift Theorem~\ref{driftthm:variable} we may bound the expected number of migrations $T_0$ before a fitness of less than $s_0$ is achieved by
\begin{align}\label{drifteq:island2}
E[T_0] \leq \frac{s_0}{h(s_0)} + \frac{1}{c \ln \lambda}\int_{s_0}^{s_1} \ln\left(\frac{n\ln \lambda}{\tau \sigma}\right) d\sigma,
\end{align}
where we used $X_0 \leq s_1$. The latter integral can now be evaluated by elementary analysis, and yields
\begin{align}\label{drifteq:island3}
\int_{s_0}^{s_1} \ln\left(\frac{n\ln \lambda}{\tau \sigma}\right) d\sigma =  \frac{\tau}{n\ln \lambda}\Big[\sigma(1-\ln \sigma)\Big]_{\tau s_0/(n\ln \lambda)}^{\tau s_1/(n\ln \lambda)},
\end{align}
from which the authors can compute their runtime bounds. We refrain from stating the final result since it involves several case distinction with respect to $\tau$ and $\lambda$.
\end{example}

\subsection{Multiplicative Drift}\label{driftsec:multiplicative}

A very important special case of variable drift is \emph{multiplicative drift}, where the drift is proportional to the potential. Introduced in~\cite{DoerrJW10g,DoerrJW12,doerr2013adaptive}, it has become the most widely used variant of drift analysis in evolutionary algorithms. In fact, all the examples~\ref{driftexample:RLSOM},~\ref{driftexample:RLSshortcuts},~\ref{driftexample:RLSBinVal}, and~\ref{driftexample:randomdecline} had multiplicative drift. In particular, Examples~\ref{driftexample:RLSshortcuts},~\ref{driftexample:RLSBinVal}, and~\ref{driftexample:randomdecline} show that the same limitations as for variable drift apply.

\begin{theorem}[Multiplicative Drift~\cite{DoerrJW12}, special case of Theorem~\ref{driftthm:variable}]
\label{driftthm:multiplicative}
Let $(X_t)_{t\geq 0}$ be a sequence of non-negative random variables with a finite state space $\mathcal{S} \subseteq \R_0^+$ such that $0 \in \mathcal S$. Let $s_{\text{min}} := \min(\mathcal{S} \setminus \{0\})$, let $T := \inf\{t \geq 0 \mid X_t =0\}$, and for $t \geq 0$ and $s \in \mathcal S$ let $\Delta_t(s) := E[X_{t} - X_{t+1} \mid X_t = s]$. Suppose there exists $\delta >0$ such that for all $s\in \mathcal{S}\setminus\{0\}$ and all $t\geq 0$ the drift is
\begin{align}\label{drifteq:multdrift1}
\Delta_t(s) \geq \delta s.
\end{align}
Then
\begin{align}\label{drifteq:multdrift2}
E[T] \leq \frac{1+E[\ln(X_0/s_{\text{min}})]}{\delta}.
\end{align}
\end{theorem}

We conclude this section by giving some applications of the multiplicative drift theorem.

\begin{example}[\ooea on Linear Functions]\label{driftexample:omlinear}
One of the cornerstones in the theory of evolutionary algorithms is the analysis of linear pseudo-Boolean functions $f:\{0,1\}^n \to \R$, i.e., functions of the form $f(x) = \sum_{i=1}^n w_ix_i$, where the $w_i$ are constants. To avoid trivialities, we assume that the weights are non-zero, and by symmetry of the search space we may assume that they are non-negative and sorted, $w_1 \geq w_2 \geq \ldots \geq w_n >0$. We have already seen two examples of such functions: \om in Example~\ref{driftexample:RLSOM} and \BinVal in Example~\ref{driftexample:RLSBinVal}.

To analyse how the \ooea with mutation rate $c=1/n$ (Algorithm~\ref{driftalgo:oplea} with offspring population size $\lambda =1$) minimses a linear function, a naive approach is to use the fitness as potential, $X_t := f(x^{(t)})$. Similar as for RLS on \BinVal, this yields a multiplicative drift of at least
\begin{align}\label{drifteq:omlinear1}
\Delta_t(s) \geq \Omega(s/n),
\end{align}
since the \ooea has at least a constant probability to perform an RLS step, i.e., to flip exactly one bit. Therefore, the Multiplicative Drift Theorem gives the bound
\begin{align}\label{drifteq:omlinear2}
E[T] \leq O\left(\frac{1+E[\ln(X_0/w_n)]}{\delta}\right).
\end{align}
For \om-like functions where all weights are similar, this bound is $O(n\ln n)$, which turns out to be tight. However, for other linear function like \BinVal, the bound is not tight, for the same reason as for RLS on \BinVal (Example~\ref{driftexample:RLSBinVal}). Rather, the expected runtime is $\Theta(n\ln n)$, as was first shown by Droste, Jansen, and Wegener in~\cite{DrosteJW02}

For the \om potential $\text{OM}_t := \om(x^{(t)})$ the situation is rather interesting. For functions like \BinVal, there are search points (e.g., the search point $(1,0,\ldots,0)$ where only the highest-valued bit is not optimised yet) in which the drift is negative, i.e., $E[\text{OM}_t-\text{OM}_{t+1} \mid x^{(t)} = (1,0,\ldots,0)] <0$. Nevertheless, J\"agersk\"upper showed~\cite{Jagerskupper08} by a coupling argument that bits of larger weight are more likely to be optimised, so that we still have a multiplicative drift~\cite{doerr2010drift} for all $t\geq 0$ and all $s \in \{1,\ldots,n\}$,
\begin{align}\label{drifteq:omlinear3}
\Delta_t(s) = E[\text{OM}_t-\text{OM}_{t+1} \mid \text{OM}_t = s] = \Omega(s/n),
\end{align}
from which a runtime bound $E[T] = O(n\ln n)$ follows. So this is one of the cases where it is beneficial to avoid filtrations and pointwise drift, see also the paragraph \emph{Drift Versus Expected Drift} in Section~\ref{driftsec:variants}.

The results can be tightened if one considers more carefully crafted potentials. Doerr, Johannsen, and Winzen showed~\cite{DoerrJW10g}, building on ideas from~\cite{HeYao:04:drift}, that the drift function $\varphi(x) := \sum_{i=1}^{\lfloor n/2 \rfloor} \tfrac{5}{4}x_i + \sum_{i=\lfloor n/2 \rfloor+1}^{n} x_i$ even has \emph{pointwise} multiplicative drift, i.e., for all $t \geq 0$ and all search points $x \in \{0,1\}^n$,
\begin{align}\label{drifteq:omlinear4}
E[\varphi(x^{(t)})-\varphi(x^{(t+1)}) \mid x^{(t)} = x] = \Omega(\varphi(x)/n).
\end{align}
This yields again the runtime bound $E[T] = O(n\ln n)$. Pointwise multiplicative drift giving similar runtime bounds can also be achieved by other potential functions~\cite{DoerrJW12}.

Similar techniques can also be used to show that the \ooea has still runtime $\Theta(n \ln n)$ on every linear function if the mutation rate is $c/n$ for an arbitrary constant $c$~\cite{doerr2013adaptive,witt2013tight,lengler2016drift}. However, this requires a considerably more complicated potential function which must necessarily depend on the mutation rate~\cite{doerr2012non}.
\end{example}

\begin{example}[Minimum Spanning Trees]\label{driftexample:MST}
Consider the following \emph{minimum spanning tree (MST)} problem proposed in~\cite{neumann2007randomized}. Let $G = (V,E)$ be a connected graph with $n$ vertices, $m$ edges $e_1, \ldots,e_m$, and positive integer edge weights $w_1, \ldots,w_m$. We denote by $w_{\max} := \max_i{w_i}$ the maximum weight. A bit string $x \in \{0,1\}^m$ represents a subgraph of $G$ with vertex set $V$, where the edge $e_i$ is present if and only if $x_i=1$. The fitness of a bit string is given by $f(x) = \sum_{i=1}^n w_i x_i + p(x)$, where $p(x)$ is a punishment term for non-trees that ensures to find a spanning tree quickly, and to stay within the set of spanning trees afterwards.

We consider the \oea on this problem. In~\cite{neumann2007randomized} it was shown that the algorithm quickly finds a spanning tree, so we assume for simplicity that the initial search point $x^{(0)}$ represents such a tree. We consider the potential function $X_t := \sum_{i=1}^n w_i x^{(t)}_i - w_{\text{opt}}$, where $w_{\text{opt}}$ is the weight of a minimum spanning tree. Then relying on results from~\cite{neumann2007randomized}, in~\cite{DoerrJW12} it is shown that the potential function has a multiplicative drift of
\begin{align}\label{drifteq:MST1}
\Delta_t(s) = E[X_t-X_{t+1} \mid X_t = s] \geq  \frac{s}{em^2}.
\end{align}
Hence, by the Multiplicative Drift Theorem~\ref{driftthm:multiplicative} the expected runtime (starting from a spanning tree) is at most
\begin{align}\label{drifteq:MST2}
E[T] \leq em^2(1 + \ln(mw_{\max})),
\end{align}
since the minimum potential of a non-optimal search point is at least $s_{\text{min}} \geq 1$, and since $mw_{\max}$ is an upper bound on $X_0$. It is an open question whether \eqref{drifteq:MST2} is tight, since the best lower bound is $\Omega(m^2\ln m)$~\cite{neumann2007randomized}, which is a \emph{tight} bound for RLS~\cite{reichel2010evolutionary}.
\end{example}

There are numerous other applications of the multiplicative drift theorem, including evolutionary algorithms on other problems~\cite{DoerrJ10,DoerrJW12,doerr2015optimizing,giessen2016robustness}, ant-colony optimisation~\cite{friedrich2016robustness}, island models~\cite{lissovoi2017runtime}, genetic programming~\cite{doerr2017bounding}, and estimation of distribution algorithms~\cite{friedrich2017compact}.

\section{Advanced Drift Theorems}\label{driftsec:advanced}

In this section we will review the most important developments in drift analysis in the last years, in particular lower and tail bounds, weak drift, negative drift, and population drift. Note that other than in the previous section, many advanced theorems, especially on tail bounds, make assumptions on the \emph{pointwise drift}, cf.~Section~\ref{driftsec:situation}.

\subsection{Lower Bounds}\label{driftsec:lower}

As discussed in Section~\ref{driftsec:variable}, the Variable Drift Theorem and the Multiplicative Drift Theorem only have a chance to give tight results if we have some restriction on the probability of making large jumps. From the earlier discussion on pages~\pageref{driftsec:tightness}ff, it is it clear that we get a matching lower bound for the Variable Drift Theorem if we apply the estimates~\eqref{drifteq:proofvariable2},~\eqref{drifteq:proofvariable3}, and~\eqref{drifteq:proofvariable4} only in tight cases. In particular, this is the case if $h(X_{t+1})/h(X_t)$ is always close to $1$. Following this idea, we get the following lower bound.
\begin{theorem}[Variable Drift Theorem, Lower Bound 1]\label{driftthm:variablelower1}
Let $(X_t)_{t\geq 0}$ be a sequence of non-negative random variables with a finite state space $\mathcal{S} \subseteq \R_0^+$ such that $0 \in \mathcal S$. Let $s_{\text{min}} := \min(\mathcal{S} \setminus \{0\})$, let $T := \inf\{t \geq 0 \mid X_t =0\}$, and for $t \geq 0$ and $s \in \mathcal S$ let $\Delta_t(s) := E[X_{t} - X_{t+1} \mid X_t = s]$. Suppose there is an increasing function $h: \R^+ \rightarrow \R^+$ and a constant $c\geq 1$ such that for all $s \in \mathcal{S}\setminus\{0\}$ and all $t\geq 0$ the following conditions hold.
\begin{align}\label{drifteq:variablelowerA1}
\Delta_t(s) \leq h(s),
\end{align}
\begin{align}\label{drifteq:variablelowerA2}
\frac{1}{c} \leq \frac{h(\max\{X_{t+1},s_{\text{min}}\})}{h(X_t)} \leq c.
\end{align}
Then
\begin{align}\label{drifteq:variablelowerA3}
E[T] \geq \frac{1}{c}\cdot\left(\frac{s_{\text{min}}}{h(s_{\text{min}})} + E\left[\int_{s_{\text{min}}}^{X_0} \frac{1}{h(\sigma)} d\sigma\right]\right),
\end{align}
where the expectation on the latter term is over the random choice of $X_0$.
\end{theorem}
 Note that the theorem gives a direct comparison between upper and lower bound: it says that they differ at most by a factor $c$. Despite its arguably natural form, it seems that the lower bound has never been formulated in this version in the literature,\footnote{though Feldmann and K\"otzing~\cite{feldmann2013optimizing} give bounds following the same ideas.} perhaps because it usually does not give tight leading constants. For example, consider RLS on \om as in Example~\ref{driftexample:RLSOM}. There $X_{t}$ is given by the fitness, and $h(s)= s/n$. The largest jump occurs when $X_t$ decreases from $2$ to $1$, in which case $h(X_{t+1})/h(X_t) = 1/2$. Thus the lower bound is a factor $2$ from the upper bound.

Doerr, Fouz, and Witt~\cite{DoerrFW11} have given a variant which usually gives a tighter lower bound. In fact, it gives a matching lower bound in many applications. Note, however, that the theorem has the rather strong condition that the sequence $X_{t}$ is non-increasing, see also the discussion after Theorem~\ref{driftthm:multiplicativelower}.
\begin{theorem}[{Variable Drift Theorem, Lower Bound 2~\cite{DoerrFW11}}]
\label{driftthm:variablelower2}
Let $(X_t)_{t\geq 0}$ be a sequence of non-negative random variables with a finite state space $\mathcal{S} \subseteq \R_0^+$ such that $0 \in \mathcal S$, and with associated filtration $\mathcal{F}_t$. Let $s_{\text{min}} := \min(\mathcal{S} \setminus \{0\})$, and let $T := \inf\{t \geq 0 \mid X_t =0\}$. 
Suppose there are two functions $\xi,h: \R_0^+ \rightarrow \R^+$ such that $h$ is monotone increasing, and such that for all $s \in \mathcal{S}\setminus\{0\}$ and for all $t\geq 0$ the following three conditions hold.
\begin{align}\label{drifteq:variablelowerB0}
X_{t+1} \leq X_t.
\end{align}
\begin{align}\label{drifteq:variablelowerB2}
X_{t+1} \geq \xi(X_t).
\end{align}
\begin{align}\label{drifteq:variablelowerB1}
E[X_t-X_{t+1} \mid \mathcal{F}_t, X_t = s] \leq h(\xi(s)).
\end{align}
Then
\begin{align}\label{drifteq:variablelowerB3}
E[T] \geq \frac{s_{\text{min}}}{h(s_{\text{min}})} + E\left[\int_{s_{\text{min}}}^{X_0} \frac{1}{h(\sigma)} d\sigma\right],
\end{align}
where the expectation on the latter term is over the random choice of $X_0$.
\end{theorem}
To apply Theorem~\ref{driftthm:variablelower2}, one should first choose $\xi$ such that \eqref{drifteq:variablelowerB2} is satisfied, and afterwards choose $h$ in sich a way that the composition $h\circ \xi$ is the drift,\footnote{In particular, the function $h$ in Theorem~\ref{driftthm:variablelower2} is not identical to the function $h$ in the upper bound version, Theorem~\ref{driftthm:variable}.} cf.~Example~\ref{driftexample:RLSOMlower} below.
 
We remark that Gie\ss en and Witt~\cite{giessen2017optimal} have developed a version in which the deterministic condition \eqref{drifteq:variablelowerB2} is replaced by a probabilistic condition. The exact formulation is rather technical. However, the theorem simplifies for multiplicative drift~\cite{witt2013tight}. We give here the version from~\cite{lehre2013general}, which assumes bounds on the probability that $X_t$ drops by more than a multiplicative factor. A version in which an \emph{additive} bound on $|X_t-X_{t+1}|$ is assumed can be found in~\cite{doerr2017bounding}.
\begin{theorem}[{Multiplicative Drift Theorem, Lower Bound~\cite{witt2013tight,lehre2013general}}]
\label{driftthm:multiplicativelower}
Let $(X_t)_{t\geq 0}$ be a sequence of non-negative random variables with a finite state space $\mathcal{S} \subseteq \R_0^+$ such that $0 \in \mathcal S$, and with associated filtration $\mathcal{F}_t$. Let $s_{\text{min}} := \min(\mathcal{S} \setminus \{0\})$, and let $T := \inf\{t \geq 0 \mid X_t =0\}$. 
Suppose there are two constants $0 < \beta, \delta \leq 1$ such that for all $s \in \mathcal{S} \setminus \{0\}$ and all $t \geq 0$ the following conditions hold.
\begin{align}\label{drifteq:multiplicativelower0}
X_{t+1} \leq X_t.
\end{align}
\begin{align}\label{drifteq:multiplicativelower1}
\Pr[X_{t} -X_{t+1} \geq \beta X_t \mid \mathcal{F}_t, X_t = s] \leq \frac{\beta \delta}{1+\ln(s/s_{\text{min}})}.
\end{align}
\begin{align}\label{drifteq:multiplicativelower2}
E[X_t-X_{t+1} \mid \mathcal{F}_t, X_t = s] \leq \delta s.
\end{align}
Then
\begin{align}\label{drifteq:multiplicativelower3}
E[T] \geq \frac{1-\beta}{1+\beta} \cdot \frac{1+E[\ln(X_0 /s_{\text{min}})]}{\delta}.
\end{align}
\end{theorem}
Recently, Doerr, Doerr, and K\"otzing~\cite{doerr2017static} showed that the monotonicity condition~\eqref{drifteq:multiplicativelower0} can be completely removed if~\eqref{drifteq:multiplicativelower2} is replaced by the condition that for all $s, s' \in \mathcal{S}\setminus\{0\}$ with $s' \leq s$,
\begin{align}\label{drifteq:multiplicativelower4}
E[\max\{s'-X_{t+1},0\} \mid \mathcal{F}_t, X_t = s] \leq \delta s'.
\end{align}
The authors show that this condition is satisfied for very natural processes. In particular it is satisfied for processes with multiplicative drift if the jump probability $p(s) := \Pr[X_{t+1} \leq s' \mid \mathcal{F}_t, X_t = s]$ is a decreasing function in $s$, whenever $s' \leq s$.\footnote{In other words, it should more likely to jump into the interval $[0,s']$ if you start closer to it.} This modification extends the scope of Theorem~\ref{driftthm:multiplicativelower} considerably, since many evolutionary algorithms are non-monotone processes. Moreover, it seems likely that the proof in~\cite{doerr2017static} can be extended to generalise related lower bounds, in particular the lower bound for variable drift in Theorem~\ref{driftthm:variablelower2}.

We conclude the discussion on lower bounds with an easy example to demonstrate how to apply Theorem~\ref{driftthm:variablelower2} and~\ref{driftthm:multiplicativelower}. 
\begin{example}[RLS on \om, Lower Bound]\label{driftexample:RLSOMlower}
Consider once more RLS on \om as in Example~\ref{driftexample:RLSOM}. We want to apply Theorem~\ref{driftthm:variablelower2}. Since $X_t$ decreases by at most one, we choose $\xi(s) := s-1$ to satisfy~\eqref{drifteq:variablelowerB2} as tightly as possible. Since the drift is $\Delta_t(s) = s/n$, we choose $h(s) := (s+1)/n$ so that $h(\xi(s)) = \Delta_t(s)$. Thus we obtain the lower bound
\begin{align}\label{drifteq:exampleELSOMlower1}
E[T] & \geq \frac{s_{\text{min}}}{h(s_{\text{min}})} + E\left[\int_{s_{\text{min}}}^{X_0} \frac{1}{h(\sigma)} d\sigma\right] = \frac{1}{2/n} + E\left[\int_{1}^{X_0} \frac{n}{\sigma+1} d\sigma\right], \nonumber\\
& = \frac{n}{2} + n \cdot E[\ln(X_0+1)-\ln 2],
\end{align}
which is easily seen to be at least $n\ln n -O(n)$.

Note that Theorem~\ref{driftthm:multiplicativelower} would give a less tight bound if naively applied. To satisfy~\eqref{drifteq:multiplicativelower1} for $s=2$, it would be necessary to choose $\beta \geq 1/2$, and for $s=1$ we even need $\beta \geq 1$, which renders the bound useless. However, this problem can be overcome by truncating the search space, see~\cite{doerr2017static} for details. 
\end{example}

\subsection{Tail Bounds}\label{driftsec:tailbounds}
In some cases, we would also like to understand $T$ beyond its expectation. In particular, we may want that $T$ is concentrated, i.e., we want bounds on the probability that $T$ deviates substantially from its expectation. This is desirable for at least two reasons. Firstly, it gives more concrete guarantees on $T$, for example that the algorithm will converge with a certain number of steps with $99\%$ probability. Secondly, it might also happen that the expectation is misleading. For example, consider the following variant of the Gambler's Ruin problem. A gambler starts with $1\$$, and with each game she either wins or loses $1\$$, but the probability of losing is $1/2 +1/n$, so slightly larger than the probability $1/2 -1/n$ of winning. Let $T$ be the time until she is broke, i.e. the number of games until she has no money left. Then the drift towards $0$ is $2/n$, and therefore $E[T] = n/2$ by the Additive Drift Theorem. However, it can be computed that $\Pr[T \leq 27] \geq 70\%$, which holds even for the fair game where winning and losing is equally likely. Therefore, for large $n$ the expectation $n/2$ is rather misleading since \emph{typical values} of $T$ are very different. Such discrepancies can be ruled out by concentration results.

For the standard drift theorems we need additional assumptions on $X_t$ for such concentration results to hold, with one notable exception. The following \emph{upper tail bound for multiplicative drift} holds without any further requirements, as pointed out by Doerr and Goldberg~\cite{doerr2013adaptive}. We give the simplified formulation from~\cite{DoerrJW12}. We also present the proof of Doerr and Goldberg, which is remarkably short and elegant.
\begin{theorem}[Multiplicative Drift, Upper Tail Bound~\cite{doerr2013adaptive,DoerrJW12}]
\label{driftthm:multiplicativetail}
Let $(X_t)_{t\geq 0}$ be a sequence of non-negative random variables with a finite state space $\mathcal{S} \subseteq \R_0^+$ such that $0 \in \mathcal S$. Let $s_{\text{min}} := \min(\mathcal{S} \setminus \{0\})$, and let $T := \inf\{t \geq 0 \mid X_t =0\}$. 
Suppose that $X_0 = s_0$, and that there exists $\delta >0$ such that for all $s\in \mathcal{S}\setminus\{0\}$ and all $t\geq 0$,
\begin{align}\label{drifteq:multtail1}
E[X_t - X_{t+1} \mid X_t = s] \geq \delta s.
\end{align}
Then, for all $r \geq 0$,
\begin{align}\label{drifteq:multtail2}
\Pr\left[T >  \left\lceil\frac{r+\ln(s_0/s_{\text{min}})}{\delta}\right\rceil\right] \leq e^{-r} .
\end{align}
\end{theorem}
\begin{proof}
For every fixed $\rho = \lceil\frac{r+\ln(s_0/s_{\text{min}})}{\delta}\rceil \in \N$, by Markov's inequality,
\begin{align}\label{drifteq:multtail3}
\Pr[T> \rho] = \Pr[X_\rho >0] \leq \frac{E[X_\rho]}{s_{\text{min}}} \stackrel{(*)}{\leq} (1-\delta)^{\rho}\frac{s_0}{s_{\text{min}}},
\end{align}
where (*) comes from applying equation~\eqref{drifteq:multtail1} and linearity of expectation $\tau$ times. Since $(1-x) \leq e^{-x}$ for all $x\in \R$, we obtain $\Pr[T> \rho] \leq e^{-\rho\delta}s_0/s_{\text{min}} \leq e^{-r}$. 
\end{proof}

For all other main drift theorems, including additive drift, variable drift, and lower tails for multiplicative drift, we need assumptions on the probability of large jumps. For example, consider the process on $\mathcal{S} = \{0,n\}$ in which $X_t=n$ has probability $1/n$ to jump to zero, and stays in $n$ otherwise. Then $X_t$ has drift one towards $0$, but the hitting time $T$ is geometrically distributed. In particular, $T$ is not concentrated.\footnote{For example, $\Pr[T > 2 E[T]] = (1-1/n)^{2n} \approx e^{-2}$.} So we need to make some assumption on the distribution of $|X_t-X_{t+1}|$. 

The easiest assumption is that large jumps do not occur at all, i.e. $|X_{t+1} - X_t| < c$ for some parameter $c$. This case occurs in various situation, for example for RLS, for some ant colony optimisation algorithms like the max-min ant system MMAS, or for the compact genetic algorithm cGA. We refer the reader to K\"otzing~\cite{kotzing14concentration} for a large collection of additive drift theorems with this assumption.

While there are situations without large jumps, there are even more cases in which large jumps may occur, but are unlikely. Thus research has focused on drift theorems with assumptions on the jump probability, usually some type of exponentially falling bounds, i.e., $\Pr[|X_{t+1} - X_t| > j] \leq c\cdot (1+\eta)^{-j}$ for some parameters $c,\eta >0$. In this chapter we stick with this type of condition, although generalisations are possible. K\"otzing has made the point that exponentially falling jump probabilities imply a sub-Gaussian distribution of $X_t -\eps t$, which is sufficient to derive most known tail bounds~\cite{kotzing2016concentration}.\footnote{and arguably more natural, using the Azuma-Hoeffding inequality.} Lehre and Witt have given a very general framework for drift theorems~\cite{lehre2013general,lehre2014concentrated}, in which only weak conditions on the exponential probability generating function $e^{\lambda (X_t-X_{t+1})}$ are needed.\footnote{more precisely, only the \emph{expectation} of this function needs to be bounded.} Most major drift theorems, including concentration bounds, can be derived from this framework, so that it arguably renders the other drift theorems unnecessary~\cite{lehre2013general}. However, researchers have continued to use specialised drift theorems, possibly because the framework by Lehre and Witt comes with a substantial technical overhead. We give their main theorem at the end of the section for quick reference, but discussing the relation to the other drift theorems is beyond the scope of this paper, and we refer the reader to the very nice exposition in~\cite{lehre2013general}.

Even with bounds on the probability of making jumps, lower tail bounds remain rather delicate. Unfortunately, it is \emph{not} true in general that the runtime is concentrated around the expectation. This problem occurs when the drift is too weak, as the following counterexample shows.

\begin{example}[Runtime is Not Concentrated Around Mean for Weak Drift]\label{driftexample:weakdrift}
We consider the following artificial random walk on the set $\{0,1\ldots,N\}$ for some (very large) constant $N$. We start in $X_0 = n$, where $n$ is much smaller than $N$. For $X_t=s$, with probability $1/n^{4}$ we make a step to the left, $X_{t+1} := X_{t}-1$, and otherwise we flip an unbiased coin to see whether we make a step to the left or to the right. We say that we do a \emph{biased step} in the first case, and an \emph{unbiased step} in the second.\footnote{We have neglected the border case $X_t = N$ in the description. However, if $N$ is large enough, e.g., $N = e^n$, then we cannot hit the right border in $o(N)$ steps, so the arguments are unaffected by the right border. For equation~\eqref{drifteq:multconc3} we need that the drift is also $1/n^4$ at the border.} Effectively, this process can be summarised as
\begin{align}\label{drifteq:multconc1}
X_{t+1} = \begin{cases} X_{t}-1& \text{ with probability $\tfrac12(1+1/n^4)$},\\  X_{t}+1 & \text{ with probability $\tfrac12(1-1/n^4)$}.\end{cases}
\end{align}
Then the drift is easily seen to be 
\begin{align}\label{drifteq:multconc2}
\Delta_{t}(s) = \frac{1}{n^4},
\end{align}
so that by the Additive Drift Theorem~\ref{driftthm:additive} we obtain
\begin{align}\label{drifteq:multconc3}
E[T] =  n^4.
\end{align}
So in terms of expectations, drift analysis can handle the problem quite well. However, it turns out that the expectation is completely misleading. Consider the first $n^3$ steps of the algorithm. By a union bound, with probability $1-O(1/n)$ all of these steps are unbiased. Hence, with high probability the first $n^3$ steps are given by an unbiased random walk, also known as a \emph{Gambler's Ruin Process}. This process is well-studied, and it is known that the probability to walk from $n$ to $0$ in at most $\alpha n^2$ steps is $1-O(\alpha^{-1/2})$ for all $\alpha >1$~\cite{grimmett2001probability}. In particular, with $\alpha = n$, the probability that an unbiased random walk starting in $n$ hits $0$ in at most $n^3$ steps is $1-O(n^{-1/2})$. Thus, with high probability the stopping time $T$ of our process satisfies $T= O(n^3)$.\footnote{In fact, being mathematically sloppy the ``typical case'' is $T=\Theta(n^2)$.} Hence, with high probability $T$ is asymptotically much smaller than its expectation $E[T] = n^4$.
\end{example}

This example is rather prototypical for situations with weak drift. In fact, it was shown in~\cite{doerr2018bounding} that in general\footnote{under some weak assumptions, in particular assuming that large step sizes are unlikely as in~\eqref{drifteq:addtail1} below.} for weak additive drift the value of $E[T]$ is not dominated by ``typical'' cases, but that at least a constant proportion of $E[T]$ comes from exceptional case in which $T$ is much larger than $E[T]$. We also remark that Example~\ref{driftexample:weakdrift} above can easily be adapted to multiplicative drift, e.g., by making the probability of an unbiased step $X_t/n^{10}$. Since $X_t$ changes in each step by at most one, by Theorem~\ref{driftthm:multiplicativelower} the bound $E[T] = O(n^{10}\log n)$ given by the Multiplicative Drift Theorem~\ref{driftthm:multiplicative} is tight up to constants factors. However, as before the runtime is $O(n^3)$ with high probability, so that with high probability the runtime is much smaller than the expected runtime.

Despite this problem, good tail bounds for additive drift have been developed. The following theorem follows by combining Theorems 10, 12, and 13 in~\cite{kotzing2016concentration}.\footnote{actually, the statement in \cite{kotzing2016concentration} is stronger since it states that \emph{at no point} during the whole process $X_t$ deviates substantially from its expectation, while we only consider $X_t$ that are relevant for the runtime.}
\begin{theorem}[Additive Drift, Tail Bounds, following~\cite{kotzing2016concentration}]
\label{driftthm:additivetail}
Let $(X_t)_{t\geq 0}$ be a sequence of non-negative random variables with a finite state space $\mathcal{S} \subseteq \R_0^+$ such that $0 \in \mathcal S$, and with associated filtration $\mathcal{F}_t$. Let $s_{\text{min}} := \min(\mathcal{S} \setminus \{0\})$, and let $T := \inf\{t \geq 0 \mid X_t =0\}$. 
Suppose that $X_0 = s_0$, and that there exist $\delta, \eta, r >0$ such that for all $s\in \mathcal{S}\setminus\{0\}$, all $j\in \N_0$, and all $t \geq 0$ the following conditions hold.
\begin{align}\label{drifteq:addtail1}
\Pr[|X_{t+1} - X_t| > j\mid \mathcal{F}_t] \leq \frac{r}{(1+\eta)^{j}}.
\end{align}
\begin{align}\label{drifteq:addtail2}
E[X_t - X_{t+1} \mid \mathcal{F}_t, X_t = s] \leq \delta.
\end{align}
Then, for all $x \geq 0$
\begin{align}\label{drifteq:addtail3}
\Pr\left[T \leq \frac{s_0-x}{\delta}\right] \leq \exp\left\{-\frac{\eta x}{8}\cdot \min\Big\{1, \frac{\eta^2\delta x}{32rs_0}\Big\}\right\} .
\end{align}
If instead of \eqref{drifteq:addtail2} we have 
\begin{align}\label{drifteq:addtail4}
E[X_t - X_{t+1} \mid \mathcal{F}_t, X_t = s] \geq \delta,
\end{align}
then
\begin{align}\label{drifteq:addtail5}
\Pr\left[T \geq \frac{s_0+x}{\delta}\right] \leq \exp\left\{-\frac{\eta x}{8}\cdot \min\Big\{1, \frac{\eta^2 \delta x}{32rs_0}\Big\}\right\} .
\end{align}
\end{theorem}
Note that the bounds in Theorem~\eqref{drifteq:addtail3} and~\eqref{drifteq:addtail5} give only concentration if the right hand side is of the form $\exp\{-\Phi\}$ for a large term $\Phi$. In particular, consider the case that $\delta$ and $r$ are constants, and that $x = \Theta(s_0)$. Then $\Phi = \omega(1)$ if and only if the bound $s_0/\delta$ on the expected runtime satisfies $s_0/\delta = o(x^2)= o(s_0^2)$. On the other hand, for $s_0/\delta = \omega(s_0^2)$ the runtime bound from the drift is larger than the time that an unbiased random walk would need to hit $0$, cf. also Example~\ref{driftexample:weakdrift}. So it is not surprising that Theorem~\ref{driftthm:additivetail} does not give concentration in this regime. Tight concentration bounds for the regime of weak drift can be found in~\cite{kotzing2016concentration}.

We conclude the section by the tail bounds in the general framework of Lehre and Witt~\cite{lehre2013general,lehre2014concentrated}. Note that~\cite{lehre2013general,lehre2014concentrated} both contain also several corollaries that correspond to simplified special cases, in particular some cases which resemble more closely our variant of the Variable Drift Theorem.

\begin{theorem}[General Drift Theorem, Tail Bounds, adapted from~\cite{lehre2014concentrated}]\label{driftthm:generalDrift}
Let $a \geq 0$, let $(X_t)_{t\geq 0}$ be a sequence of random variables with a finite state space $\mathcal{S} \subseteq \R_0^+$ such that the interval $[0,a] \cap S$ is absorbing, and with associated filtration $\mathcal{F}_t$. Let $T_a := \inf\{t \geq 0 \mid X_t \leq a\}$, and assume $X_0 = s_0 >a$.
Moreover let $\lambda>0$, let $g: \R_0^+ \rightarrow \R_0^+$ be a function such that $g(0)=0$ and $g(s)\geq g(a)$ for all $s>a$, and let $\beta: \N \to \R^+$.
\begin{enumerate} 
\item If for all $t\geq 0$,
\begin{align}\label{drifteq:genthm1}
E[e^{-\lambda(g(X_t)-g(X_{t+1}))} \mid \mathcal F_t, X_t > 0] \leq \beta(t),
\end{align}
 then for all $t \geq 0$,
\begin{align}\label{drifteq:genthm2}
\Pr[T_a > t] < \left(\prod_{r=0}^{t-1}\beta(r)\right)e^{\lambda(g(s_0)-g(a))}.
\end{align}
\item If for all $t\geq 0$,
\begin{align}\label{drifteq:genthm3}
E[e^{\lambda(g(X_t)-g(X_{t+1}))} \mid \mathcal F_t, X_t > 0] \geq \beta(t),
\end{align}
 then for all $t \geq 0$,
\begin{align}\label{drifteq:genthm4}
\Pr[T_a < t] \leq \left(\prod_{r=0}^{t-1}\beta(r)\right)e^{-\lambda (g(s_0)-g(a))}.
\end{align}
\end{enumerate}
\end{theorem}

In general, in order to obtain tail bounds for variable drift, we can either apply Theorem~\ref{driftthm:generalDrift}. Or alternatively we can rescale $X_t$, as discussed in Section~\ref{driftsec:variable}, to turn variable drift into additive drift, and then apply Theorem~\ref{driftthm:additivetail}. Unfortunately, both approaches tend to be considerably technical. The most important case is to obtain tight lower tail bounds for multiplicative drift. Even with the framework of Lehre and Witt, in order to derive lower tail bounds for the \ooea on \om, it is still necessary to split the process into phases of relatively constant drift~\cite{lehre2013general}. An easy and comprehensive lower tail bound for multiplicative drift is yet missing in the literature.


\subsection{Negative Drift}\label{driftsec:negative}

If the drift does not point towards zero, but rather it points with a constant rate away from zero, then it takes exponential time to cross an interval. The first theorem of this type was proven by Oliveto and Witt~\cite{Oli-Wit:j:11:negativeDrift,oliveto2012erratum}, following Hajek's classical work~\cite{hajek1982hitting}. We give a formulation close to~\cite{rowe2014choice,lengler2016drift} because it avoids $o$-notation for the length of the interval. Explicit constants can be found in~\cite{oliveto2015improved,kotzing2016concentration,Witt2017upper}.

\begin{theorem}[Negative Drift, following~\cite{Oli-Wit:j:11:negativeDrift,oliveto2012erratum,rowe2014choice,lengler2016drift}]\label{driftthm:negativeDrift}
For all $a,b,\delta,\eta,   r  >0$, with $a<b$, there is $c>0, n_0 \in \N$ such that the following holds for all $n\geq n_0$. Suppose $(X_t)_{t\geq 0}$ is a sequence of random variables with a finite state space $\mathcal{S} \subseteq \R_0^+$, and with associated filtration $\mathcal{F}_t$. Assume $X_0\geq bn$, and let $T_a := \min\{t \geq 0 \mid X_t \leq an\}$ be the hitting time of $\mathcal{S} \cap [0,an]$. Assume further that for all $s\in \mathcal{S}$ with $s>an$, for all $j\in \N_0$, and for all $t \geq 0$ the following conditions hold.
\begin{align}\label{drifteq:negthm1}
E[X_{t}-X_{t+1} \mid \mathcal{F}_t, X_t =s ] \leq - \delta.
\end{align}
\begin{align}\label{drifteq:negthm2}
\Pr[|X_{t}-X_{t+1}| \geq j \mid \mathcal{F}_t, X_t =s] \leq \frac{r}{(1+\eta)^{j}}.
\end{align}
Then 
\begin{align}\label{drifteq:negthm3}
\Pr[T_a \leq e^{cn}] \leq e^{-cn}.
\end{align}
\end{theorem}

Negative drift is helpful for proving lower bounds~\cite{rowe2014choice,oliveto2015improved,lengler2016drift}, but not only so. It may also be used to show that an algorithm stays in a desired parameter regime. For example, Neumann, Sudholt, and Witt used it to show that an Ant Colony Optimisation (ACO) algorithm has good runtime because all pheromone values stay in a desirable range~\cite{neumann2010few}. Similarly, K\"otzing and Molter~\cite{kotzing2012aco}, as well as subsequent work~\cite{lissovoi2015runtime,friedrich2016robustness,lissovoi2016mmas} used negative drift to show that ACO algorithms tend to stay close to the optimum, thus enabling the algorithm to follow the optimum in a dynamically changing environment. In a different setting, Sudholt and Witt~\cite{sudholt2016update} show that the compact Genetic Algorithm cGA is efficient on \om\footnote{in some parameter regimes} because for each position the probability to sample a one-bit never becomes too low. Similar ideas have been applied for population-based non-elitist algorithms in the Strong Selection Weak Mutation (SSWM) regime~\cite{paixao2017towards}.

\subsection{Populations}\label{driftsec:populations}
If the algorithm uses population sizes larger than one, or if it does not work at all with populations, like Ant Colony Optimisation (ACO) or Estimation of Distribution Algorithms (EDAs),\footnote{ACO algorithms maintain pheromone values, EDAs maintain a probability distribution, rather than a population of search points.} then it is often challenging to find a single potential $X_t$ which captures well the quality of the current population. As before, \emph{if} such a potential can be found then drift analysis can take care of the rest. In some cases, it suffices to consider the current best optimum as potential (Example~\ref{driftexample:island},~\cite{doerr2017island}), or some average quality~\cite{friedrich2015benefit,sudholt2016update}. A systematic approach was developed by Corus, Dang, Eremeev, and Lehre~\cite{corus2014level,corus2017level}, who gave the so-called Level-Based Theorem for population-based algorithms. With this theorem, they have identified a generic situation in which a good potential can be found automatically. A population-based algorithm in their sense\footnote{conflicting terminology exists.} is any algorithm of the following form. In each round it maintains a population of size $\lambda$, and from this population it generates some probability distribution $\mathcal{D}$. For the next round, it produces independently $\lambda$ samples from $\mathcal D$, which form the next generation.

This framework of population-based algorithms applies to many situations, often with a twist to the usual algorithm description. Firstly, it does include all $(\mu,\lambda)$-evolutionary or genetic algorithms if the $\lambda$ offsprings are generated independent of each other. In this case, let $P_i$ be the $i$-th offspring population.\footnote{\emph{not} the parent population, since from these the next parents are not sampled independently. Rather, the parents of the next generation need to compete with each other in the selection step.} Then from $P_i$ a complex process determines some probability distribution $\mathcal{D}$ from which the next offspring is sampled. This process subsumes selection and mutation/crossover. Other population-based algorithms include Simulated Annealing, and, surprisingly, EDAs~\cite{dang2015simplified}. While these latter algorithms conceptually maintain a probability distribution rather than a population, they do produce a sample population in each round, from which the next distribution is computed. This offspring population makes them fit into the framework of population-based algorithms. 
 
The Level-Based Theorem assumes a partitioning of the search space into fitness levels that need to be climbed by the population. It gives an upper bound on the expected runtime if certain conditions are satisfied. The exact formulation is rather technical, so we refer the reader to~\cite{corus2017level}. Qualitatively, three ingredients are required:
\begin{enumerate}
\item If part of the population has at least fitness level $i$, then the probability to sample an offspring at level $i+1$ is sufficiently large. 
\item The fraction of the population which has fitness level at least $i$ increases in expectation.
\item The population size is large enough.
\end{enumerate}
Although it was only recently developed, the Level-Based Theorem has already found quite a number of applications, including the analysis of genetic algorithms with a multitude of selection mechanisms and benchmark functions~\cite{corus2017level}, EDAs~\cite{dang2015simplified,lehre2017improved}, the analysis of self-adaptive algorithms~\cite{dang2016self}, and of algorithms in situations that are dynamic~\cite{dang2017populations}, noisy~\cite{dang2015efficient}, or provide only partial information~\cite{dang2016runtime}. 
%
%
%

\section{Finding the Potential Function}\label{driftsec:potential}

At the very beginning of the chapter, we have listed three ingredients for runtime analysis via drift theory: finding a good potential function, computing the drift, and transferring the knowledge about the drift into information about the runtime. In this chapter, we have discussed the third point, because it is based on a universal technique that applies to many settings. In contrast, the first two points are highly problem-dependent, and cannot be generalised well. As mentioned before, the second point is usually not the hardest part, though it is often the most technical part and sometimes tedious. On the other hand, the first task -- finding a good potential function -- is often the hardest part, and it requires a lot of insight into the problem. Unfortunately, it is difficult to give general advice on how to find an appropriate fitness function for a given problem. Nevertheless, we will try to give some approaches which may be helpful.

A first question may be whether drift analysis is always applicable, or whether there are cases where the method fails completely. More concretely: is there always a good potential function, ideally one with constant drift? The answer is pleasantly clear: ``Yes''. In theory, there is even a surprisingly simple answer to the question how this potential may look like. We may always choose the \emph{canonical potential} $X_t := E[T \mid \mathcal F_t]-t$, where $\mathcal F_t$ is the history of the algorithm up to time $t$. Note that $T$ is as usual the total number of steps of the process, it is not just the number of remaining steps. For the canonical potential we always get a drift of exactly $1$, for rather trivial reasons~\cite{HeYao:04:drift,DoerrJW12}. The canonical potential does not look very helpful, since it seemingly only helps finding the runtime if we already know the runtime. 
However, the canonical potential gives us a natural candidate for the right potential function if we have any \emph{guess} on what the runtime might be. The guess may come from heuristic considerations or from simulations. The situation resembles induction, where finding the right induction hypothesis is sometimes much harder than actually proving the inductive step. With the random decline in Example~\ref{driftexample:randomdecline} we have already seen a case where the situation was obscure, but after the right scaling $Y_t = 1+ \ln(X_t)$ it was rather easy to check that the drift is constant. How do we get to such a scaling? Re-inspecting the example, we find that it is very natural to guess that the runtime is logarithmic, so a scaling of the form $Y_t = c_1 +c_2\log (X_t)$ is a natural candidate. Indeed, every scaling of this form would have been sufficient. Choosing $c_1 = c_2 = 1$ was just the most convenient choice, due to the fact that then $Y_t = 1$ if and only if $X_t = 1$. We have seen other examples of the rescaling technique in Example~\ref{driftexample:RLSLO} and in the Variable Drift Theorem.

Note that the canonical potential is more than ``just'' a rescaling technique, since it defines $X_t$ from scratch. In particular, we can theoretically compute the expected runtime for \emph{every random process}\footnote{if the expected runtime is finite. However, the process does not need to have finite, or bounded, or discrete search spaces.} by drift analysis, by using the canonical potential. In practice, the main problem is that the history $F_t$ (or even the current state) is too complicated to work with, and likewise the canonical potential is often too complex too handle. Therefore, the art of drift analysis lies in finding a potential which is simple and manageable, but which still resembles the canonical potential.\smallskip

Let us consider next the (quite realistic) scenario that we already have some candidate for a potential function, but that this candidate is still not good enough. Let us first discuss what it means that that a potential function is ``not good''. If we want additive drift, it means that there are different states $s_1,s_2$ of the algorithm with very different drift. If we want multiplicative drift, it means that the ratio between drift and potential is very different for some states $s_1,s_2$, because we want the drift to be proportional to the potential. So the first task is to look for states with such discrepancies. Then we can try to repair this defect: if the drift at $s_1$ is too large (compared to the drift at $s_2$) then we must try to decrease the difference of the potential of $s_1$ and the potentials of typical successor states of $s_1$. We can do this either by decreasing the potential of $s_1$, or by increasing the potentials of successor states. Hopefully, this will improve the accuracy of the potential function. We may iterate this procedure until we arrive at a good potential function.

For concreteness, let us study this approach for an example. Consider the \ooea with standard mutation rate $1/n$ for minimising \BinVal, where $\BinVal(x) = \sum_{i=1}^n 2^{n-i}x_i$. Our first guess is to use the fitness function as potential, $X_t := \BinVal(x^{(t)})$. Our hope is that we get multiplicative drift, as we got for RLS in Example~\ref{driftexample:RLSBinVal}. However, with this potential we have two problems. Firstly, the potential may make huge jumps (e.g., decrease from $2^{n-1}$ to $0$), so we should be careful when applying the multiplicative drift theorem, as we saw in Examples~\ref{driftexample:RLSshortcuts} and~\ref{driftexample:RLSBinVal}. Secondly, the drift is not very close to multiplicative. For example, consider the search points $s_1 := (0,\ldots,0,1)$ and $s_2 := (1,0,\ldots,0)$. The potentials are $x_1 := 1$ and $x_2:=2^{n-1}$, respectively. A mutation of $s_1$ is only accepted if it flips the last bit, and no other bit, which happens with probability $\approx 1/(en) = x_1/(en)$. On the other hand, for $s_2$ we accept every mutation that flips the first bit, which happens with probability $1/n$. We may also flip a few other bits, but the potential still goes down by $(1-o(1))2^{n-1}$ in expectation if we flip the first bit. Therefore the drift is $\approx 2^{n-1}/n = x_2/n$. So the drift for $s_2$ is by factor $e$ larger than desired, if we compare it to $s_1$. Hence, we should try to decrease the potential for $s_2$ and/or increase it for $s_1$. A natural way to do this is to reduce the weight of the higher-order bits for computing the potential. This might also alleviate the effects of large jumps.

How much should we reduce the weight? In the extreme case, we would make all weights equal, i.e., we would use the \onemax potential. This works well on strings where the higher-order bits are all zero. For example, for $s_2$ we get a drift of $1/n$, and the ratio between drift and fitness is generally very close to $1/n$ if all one-bits are in the last, say, 10\% of the string. However, there is a problem for $s_1$. Here we accept an offspring whenever we flip the first bit, and in this case we flip an expected number of $(n-1)/n$ other bits. Therefore, the drift for $s_2$ is $1/n \cdot (1- (n-1)/n) = 1/n^2$, so it is too small. Hence we should increase the potential of $s_2$ compared to the potential of its typical offspring, i.e., we should increase the weight of higher-order bits. It takes some fiddling to get the right tradeoff, but Doerr, Johannsen, and Winzen figured out that a good choice is a weight of $5/4$ for the first half of the bits, and of $1$ for the second half of the bits~\cite{DoerrJW10g}. This choice works not only for \BinVal, but for all linear functions, cf.~Example~\ref{driftexample:omlinear}. In principle, the same approach can also be used for other mutation rates than $1/n$. In one of the most important results on the theory of evolutionary algorithms, for a mutation rate of the form $c/n$ for any constant $c>0$ Witt~\cite{witt2013tight} managed to find weights which lead to a good potential. In this way, he could prove in just $2$-$3$ pages that the runtime of the \ooea is $(1\pm o(1))\tfrac{e^{c}}{c} n \ln n$, settling a question that had been open for years.

We should keep in mind that the methods discussed above are only guidelines, which may be helpful in some situations, but fruitless in others. Finding the right drift function often requires ingenuity, and cannot be reduced to a simple cooking recipe. Thus it is still one of the most challenging, but also most rewarding tasks in runtime analysis. 

\section{Conclusion}\label{driftsec:conclusion}

We have seen how drift analysis can be applied to transform knowledge about the drift into knowledge about the runtime of an algorithm. In this chapter we have restricted ourselves to applications in the analysis of evolutionary algorithms, but drift analysis can be applied to other randomised algorithms or random processes. We refer the reader to~\cite{gobel2018intuitive}, which contains a nice variety of applications of drift analysis, including algorithms for approximate vertex cover, 2-SAT, and random sorting, and to processes like the Moran process.

In theory it is always possible to apply drift analysis to obtain matching upper and lower bound on the expected runtime. However, in practice there are many situations which are still difficult to handle because we do not know a good potential function. In particular, the more complex the state space and the behaviour of the algorithms are, the more difficult it is to find a single real-valued function which is a sufficiently good measure of the progress. For example, in Genetic Programming (GP) the states are trees instead of strings, which makes the situation considerably more complex. In the few cases where theoretical results exist, this is mostly because the tree structure is unimportant for the problem~\cite{Durrett11computational,neumann2012computational,doerr2017bounding,kotzing2018destructiveness}, with the notable exception of~\cite{doerr2018evolution}. Similarly, while there have been impressive advances for large population sizes, especially through the Level-Based Theorem (see Section~\ref{driftsec:populations}), these techniques are still limited to some special cases of population dynamics. In particular, they only consider the number of individuals on each fitness level. This limits the complexity of interactions that we can understand with this method -- for example, the approach is blind to beneficial crossovers that happen between search points on the same fitness levels. In general, it remains a major challenge to apply drift analysis to complex state spaces, and to algorithms which maintain and utilise a large diversity within their population, for example through crossover.
 
\bibliographystyle{alpha}

\newcommand{\etalchar}[1]{$^{#1}$}

\end{document}